\documentclass[lettersize,journal]{IEEEtran}
\usepackage{amsmath,amsfonts}
\usepackage{algorithmic}
\usepackage{algorithm}
\usepackage{array}
\usepackage[caption=false,font=normalsize,labelfont=sf,textfont=sf]{subfig}
\usepackage{textcomp}
\usepackage{stfloats}
\usepackage{url}

\usepackage{verbatim}
\usepackage{graphicx}
\usepackage{cite}
\usepackage{booktabs} 
\usepackage[numbers]{natbib}  
\usepackage{bm}
\usepackage{hyperref}       
\usepackage{adjustbox}
\usepackage{multirow}
\usepackage{amsmath}
\usepackage{amssymb}
\usepackage{multicol}
\usepackage{pifont}

\usepackage{amsthm}
\usepackage{enumitem}
\newtheoremstyle{IEEEtheorem}
  {\topsep}            
  {\topsep}            
  {\itshape}           
  {}                   
  {\bfseries}          
  {.}                  
  { }                  
  {}                   

\theoremstyle{IEEEtheorem}

\newtheorem{theorem}{Theorem}

\newtheorem{lemma}{Lemma}

\newtheorem{definition}{Definition}
\newtheorem{assumption}{Assumption}
\newtheorem{remark}{Remark}

\begin{document}

\title{Provable Sparse Inversion and Token Relabel Enhanced One-shot Federated Learning with ViTs}

\author{
Li Shen, Xiaolei Hao, Qinglun Li, Xiaochun Cao, Zhifeng Hao and Xun Yang
\thanks{Li Shen, Qinglun Li, and Xiaochun Cao are with the Shenzhen Campus of Sun Yat-Sen University (email:mathshenli@gmail.com, liqinglun@nudt.edu.cn, caoxiaochun@mail.sysu.edu.cn).}
\thanks{Xiaolei Hao and Xun Yang are with the University of Science and Technology of China (email:hxl95@mail.ustc.edu.cn;xyang21@ustc.edu.cn).}
\thanks{Zhifeng Hao is with Shantou University (email:haozhifeng@stu.edu.cn).}
}




\maketitle
\begin{abstract}
One-Shot Federated Learning, where a central server learns a global model in a single communication round, has emerged as a promising paradigm. 
However, under extremely non-IID settings, existing data-free methods often generate low-quality data that suffers from severe semantic misalignment with ground-truth labels. 
To overcome these issues, we propose a novel Federated Model Inversion and Token Relabel (FedMITR) framework, which trains the global model by fully exploiting all patches of synthetic images. Specifically, FedMITR employs sparse model inversion during data generation, selectively inverting semantic foregrounds while halting the inversion of uninformative backgrounds. To address semantically meaningless tokens that hinder ViT predictions, we implement a differentiated strategy: patches with high information density utilize generated pseudo-labels, while patches with low information density are relabeled via ensemble models for robust distillation. 
Theoretically, our analysis based on algorithmic stability reveals that Sparse Model Inversion eliminates gradient instability arising from background noise, while Token Relabel effectively reduces gradient variance, collectively guaranteeing a tighter generalization bound. 
Empirically, extensive experimental results demonstrate that FedMITR substantially outperforms existing baselines under various settings.
\end{abstract}


\IEEEpeerreviewmaketitle
\section{Introduction}
\IEEEPARstart{F}{ederated} learning (FL) \citep{mcmahan2017communication} is a machine learning framework in which multiple clients collaborate to solve machine learning problems under the coordination of a central server or service provider. And the raw data of each client is stored locally and is not exchanged or transferred \citep{konecny2016federated}. In recent years, FL has shown its potential to facilitate real-world applications in many fields, including recommendation systems \citep{liang2021fedrec++, liu2021fedct}, medical image analysis \citep{liu2021feddg}, computer vision \citep{lu2023feddad, zhang2023visual}, and natural language processing \citep{deng2022secure}. However, FL poses significant challenges in terms of communication cost and heterogeneity of data between clients. Communication cost is a major bottleneck in FL systems, as clients need to communicate frequently with the server over multiple rounds during the training process. This paradigm presents significant challenges: 1) the heavy communication burden \citep{li2020federated2,li2025dfedadmm}, 2) the risk of connection drop errors between clients and the server \citep{kairouz2021advances, dai2022dispfl}, and 3) the potential risk of man-in-the-middle attacks \citep{9183938} and various other privacy or security concerns \citep{mothukuri2021survey, yin2021see}. 

\looseness=-1
One-shot FL \citep{guha2019one} has emerged as a solution to these issues by restricting communication rounds to a single iteration, thus mitigating errors arising from multi-round communication and concurrently diminishing the vulnerability to malicious interception. Furthermore, the FL one-shot framework is a particularly contemporary model for market scenarios \citep{vartak2016modeldb} where clients predominantly offer pre-trained models. The drawbacks of this method stem from its challenges when dealing with strongly non-iid data \citep{beitollahi2024parametric}. Due to its single communication session, it is unable to indirectly gather information from other clients through multiple interactions. This leads to a significantly low and unstable accuracy in the single aggregation, as each client's acquired knowledge is extremely limited.


Existing approaches typically employ federated distillation, where a client model ensemble transfers knowledge to the server while simultaneously guiding data synthesis for Data-Free Knowledge Distillation (DFKD) \citep{zhang2022dense, zhang2022fine}. However, our analysis reveals a critical flaw: under severe data heterogeneity, the knowledge embedded in local models diverges significantly. Consequently, standard generators produce low-fidelity data, frequently causing severe semantic misalignment between synthesized data and their assigned labels.

To address these challenges, we propose a novel one-shot FL framework named FedMITR, which leverages a ViT-based model inversion framework for server-side data synthesis to train the global model by better utilizing all patches of the generated images. 
Specifically, we start by synthesizing input images from random noise, without utilizing any additional information from the training data, making it suitable for FL where data privacy is crucial.  After a single communication round, we obtain only the model without any data. In a data-free scenario, our approach involves recovering training data from pre-trained client models in some manner and utilizing it for knowledge transfer. Furthermore, due to the dispersed training of client models, the knowledge from each client is limited, resulting in poor quality of synthesized data. Therefore, we need to select sparse patches with different information densities for subsequent processing (Called Sparse Model Inversion). For patches with high information density, they are likely to match pseudo-labels and can be directly used for training the global model. For patches with low information density, we also reuse them by relabeling through ensemble models for knowledge distillation (Called Token Relabel), which ensures that the knowledge contained in the generated data is maximally exploited. 
Theoretically, we provide a rigorous analysis based on algorithmic stability, proving that FedMITR achieves a strictly tighter generalization bound than traditional dense inversion. Specifically, we demonstrate that Sparse Model Inversion eliminates gradient instability caused by noise, while Token Relabel significantly reduces gradient variance, thereby optimizing the loss landscape with a smaller Lipschitz constant. 
Empirically, the experimental results indicate that FedMITR significantly improves accuracy compared to existing one-shot FL methods across various heterogeneous data scenarios. For example, FedMITR surpasses the best baselines with 3.20\%, 8.92\%, and 7.93\% on CIFAR10, OfficeHome, and Mini-Imagenet under $Dir(0.1)$ heterogeneous setting, respectively.


In summary, our main contributions are summarized as follows:
\begin{itemize}[leftmargin=10pt]
\item We pioneer the exploration of Vision Transformers in One-Shot FL and propose a novel framework, FedMITR, to address the limitations of existing DFKD methods. By integrating sparse model inversion with token relabeling, our method effectively utilizes all generated patches to synthesize high-quality pseudo-data.
\item Our framework is designed exclusively for the server side, requiring no additional training on local clients. This characteristic makes FedMITR highly suitable for contemporary model market scenarios, enabling efficient global model learning without compromising data privacy or requiring extra transmissions.
\item We provide a rigorous theoretical analysis based on algorithmic stability to validate the generalization superiority of FedMITR. We mathematically prove that our sparse inversion and token relabeling mechanisms stabilize the optimization by reducing the Lipschitz constant of loss gradients, thereby guaranteeing a strictly tighter generalization bound than traditional dense inversion.
\item Extensive empirical studies on various heterogeneous datasets demonstrate that FedMITR consistently outperforms state-of-the-art baselines. For instance, it achieves accuracy improvements of 3.20\%, 8.92\%, and 7.93\% on CIFAR10, OfficeHome, and Mini-Imagenet, respectively.
\end{itemize}
\section{Related Work}
\textbf{One-Shot Federated Learning.} 
Guha et al. \citep{guha2019one} first propose the concept of One-Shot Federated Learning, which treats local models as an ensemble for final prediction and further introduced the use of knowledge distillation along with public data for this ensemble in a single round of communication.
Zhou et al. \citep{zhou2020distilled} refrain from using public data and instead propose transmitting refined local datasets to the server.
Li et al. \citep{ijcai2021p205} propose a method utilizing a two-tier knowledge transfer structure FedKT for distillation on public datasets. Instead of using public data, Zhang et al. \citep{zhang2022dense} propose a data-free method for knowledge distillation by synthesizing data directly from ensemble models on the server-side. Diao et al. \citep{diao2023towards} and Heinbaugh et al. \citep{heinbaugh2023datafree} modify the local training phase and by introducing placeholders or conditional variational autoencoders require additional transmissions. Yang et al. \citep{yang2024exploring} suggest using auxiliary pre-trained diffusion models. Previous works either required additional transmission of information from clients or trained simple generators to synthesize low-quality data, which cannot cope with one-shot FL settings in extremely heterogeneous environments.

\textbf{Model Inversion.} Fredrikson et al. \citep{fredrikson2015model} introduce model inversion attack to reconstruct private inputs. Subsequent works broaden this approach to new attack scenarios \citep{he2019model, yang2019adversarial}. Morerecently, model inversion has been used in data-inaccessible scenarios for tasks like data-free knowledge transfer \citep{yu2023data, braun2024deep, patel2023learning}.
DeepInversion \citep{yin2020dreaming} improves synthetic data with batch norm distribution regularization for visual interpretability. 
However, previous studies don't utilize model inversion in FL, and it's merely used as a tool for synthesizing surrogate data. Our work is the first to apply it to one-shot FL and enhance it to obtain generated data that can be better utilized.

\textbf{Vision Transformer.} ViT \citep{dosovitskiy2020image} is one of the earlier attempts that achieved state-of-the-art performance on ImageNet classification, using pure transformers as basic building blocks \citep{vaswani2017attention}. 
DeiT \citep{touvron2021training} manages to tackle the data-inefficiency problem by simply adjusting the network architecture and adding an additional token along with the class token for Knowledge Distillation to improve model performance. In this paper, we focus on using existing ViT models to distinguish high and low information density tokens, aiming for better knowledge transfer from ensemble models to the global model.

The most related works to ours are the DENSE \citep{zhang2022dense} and DeepInversion \citep{yin2020dreaming}. In one-shot FL, we applied a new model inversion method for data synthesis, which diverges from traditional DFKD approaches. Traditional methods, such as DENSE \citep{zhang2022dense}, generate synthetic data for distillation by training generators. In contrast, we do not require training generators but instead directly use local models to invert and synthesize data. Unlike \citep{yin2020dreaming}, which uses entire images generated by inversion for knowledge transfer, our method distinguishes tokens into high information density tokens and low information density tokens and relabels the latter for better utilization of the synthetic data.
\section{Rethinking Data-Free Method in One-Shot FL}

\looseness=-1
In this section, we first review the basic process of One-Shot Federated Learning. Then, we rethink the shortcomings of existing data-free methods in synthesizing pseudo data for FL.

\begin{figure*}[!t]
  \centering
  \includegraphics[width=1.0\textwidth]{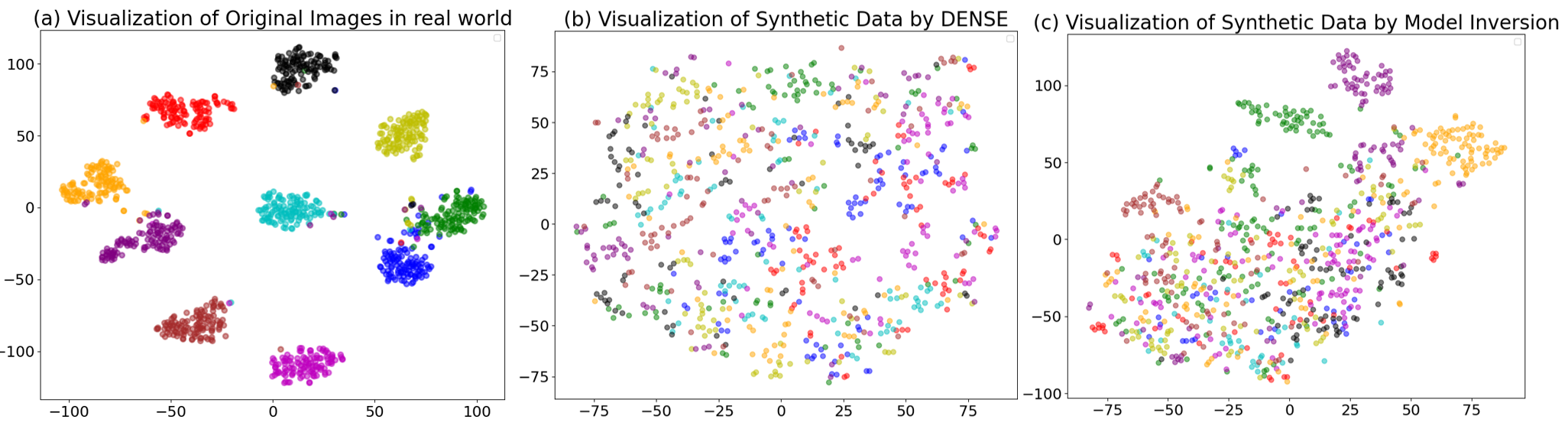}
  \caption{t-SNE visualisation of the features. (a)-(c) represent the feature distribution visualizations for the original training images, the synthetic images using traditional methods, and the synthetic images using model inversion methods on CIFAR10, respectively.}
    \label{fig_rethink}
    \vspace{-1em}
\end{figure*}

\subsection{Preliminary}
\label{Pre}
\looseness=-1
We focus on the centralized setup that consists of a central server and a set of clients $\mathbb{C}$, with $N=|\mathbb{C}|$ clients owning private labeled datasets $\mathbb{D}=\{(\bm{X}_i,\bm{Y}_i)\}_{i=1}^{N}$ in total, where $\bm{X}_i=\{(\bm{x}^k_i)\}_{k=1}^{n_i}$ follows the data distribution $\mathcal{D}_i$ over feature space $\mathcal{X}_i$, i.e., $\bm{x}^k_i \sim \mathcal{D}_i$ and $\bm{Y}_i=\{(\bm{y}^k_i)\}_{k=1}^{n_i}$ denotes the ground-truth labels of $\bm{X}_i$. The goal of one-shot federated learning is to train a good machine learning model $\bm{f}_S(\cdot)$ with parameter $\bm{\theta}_S$ over $\mathbb{D}\triangleq\cup_{i=1}^N\mathbb{D}_i$ in only one communication, as in
\begin{equation}\label{fl_loss}
    \min_{\bm{\theta}_S \in \mathbb{R}^d} \mathcal{L}(\theta_S)\triangleq\frac{1}{|\mathbb{D}|}\sum_{i=1}^N
    \mathbb{E}_{(\bm{x}_i,\bm{y}_i) \sim \mathcal{D}_i}
    [\ell(f_S(\bm{x}_i;\bm{\theta}_S),\bm{y}_i))],
\end{equation}
where $\ell(\cdot,\cdot)$ is the loss function, $f_S(\bm{x}_i;\bm{\theta}_S)$ is the prediction function of the server that outputs the logits (i.e., outputs of the last fully connected layer) of $\bm{x}_i$ given parameter $\bm{\theta}_S$ and $y_i$ denotes the corresponding one-hot label of $\bm{x}_i$.

In FL, the global model is updated by averaging the model parameters from different clients during training. However, this can only be done directly if all the models have the same structure and size, which can be a restrictive constraint in many cases. Additionally, in real-world scenarios, the data distributions across different clients may be Non-IID (Non-Independent and Identically Distributed) or subject to domain shifts. As a result, the global model obtained by averaging model parameters tends to have poor generalization performance.
For one-shot FL, it is crucial to aggregate multiple local models into a single global model. Ensemble learning allows combining multiple heterogeneous weak classifiers by averaging the predictions of individual models. In FL, the original training set $\mathbb{D}_i$ cannot be accessed, and only well-pretrained models $\bm{f}_i(\cdot)$ parameterized by $\bm{\theta}_i$, are provided. Here, we define the Ensemble $\bm{E}_S(\cdot)$ as:
\begin{equation}
\label{ensemble}
    \bm{E}_S(\bm{x};\{\bm{\theta}_i\}_{i=1}^N)\triangleq\sum_{i=1}^Nw_i\bm{f}_i(\bm{x};\bm{\theta}_i),
\end{equation}
where $\bm{f}_i(\bm{x};\bm{\theta}_i)$ is the prediction function that output the logits of $\bm{x}$ given the model $\bm{\theta}_i$, while $\bm{w}=[w_1,w_2,..,w_N]$ adjusts the weights of each local client logits. Typically we set $w_i=1/N$, especially for the server that do not know the number of data points for each client. And We use $\bm{E}_S(\bm{x})$ to denote $\bm{E}_S(\bm{x};\{\bm{\theta}_i\}_{i=1}^N)$, which means the output logits of the Ensemble given $\bm{x}$.

The first step is to train the auxiliary generator during the data generation phase.
When aggregating pre-trained models $\{\bm{\theta}_i\}_{i=1}^N$ into one server model $\bm{\theta}_S$, we aim to train a generator to generate synthetic data $\mathbb{D}_S$ with the data distribution $\mathcal{D}_S$ based on the Ensemble output. In particular, giving a random noise $\bm{z}$ generated from a standard Gaussian distribution and a random uniformly sampled one-hot label $\hat{y}$, the generator $G(\cdot)$ with parameter $\theta_G$ is responsible for generating the data $\hat{\bm{x}}=G({\bm{z}})$, forming the synthetic dataset $\mathbb{D}_S$. Since we are unable to access the training data of clients, we cannot compute the similarity between the synthetic data and the training data directly. Typically, to make sure the synthetic data can be classified correctly with a high probability by the Ensemble $\bm{E}_S(\cdot)$, as in:
\begin{equation}\label{gen_loss_old}
    \min_{\bm{\theta}_G \in \mathbb{R}^d} \mathcal{L}(\bm{\theta}_G)\triangleq\frac{1}{|\mathbb{D}_S|}\sum
    \mathbb{E}_{\hat{\bm{x}} \sim \mathcal{D}_S} 
    [\ell_{CE}(\bm{E}_S(\hat{\bm{x}}),\hat{y})],
\end{equation}
where $\ell_{CE}(\cdot,\cdot)$ denotes the cross-entropy function. In addition, various existing data-free methods often incorporate additional loss functions to train generators to ensure the quality of generated data.

After getting the synthetic dataset $\mathbb{D}_S$ based on the well-trained generator in Eq.(\ref{gen_loss_old}), existing federated distillation methods intends to distill the ensemble $\bm{E}_S$ into the final server model $\theta_S$ with the help of these synthetic data, as in:
\begin{equation}
\label{kd_loss_old}
\small
    \min_{\bm{\theta}_S \in \mathbb{R}^d} \mathcal{L}(\bm{\theta}_S)\triangleq \frac{1}{|\mathbb{D}_S|}\sum 
    \mathbb{E}_{\hat{\bm{x}} \sim \mathcal{D}_S}
    [\ell_{KL}(\bm{E}_S(\hat{\bm{x}}),\bm{f}_S(\hat{\bm{x}};\bm{\theta}_S))],
\small
\end{equation}
where $\ell_{KL}(\cdot,\cdot)$ denotes the Kullback-Leibler (KL) divergence.

\subsection{Rethinking the Way of Synthesizing Data in FL}
\label{old_syn}
\textbf{How do traditional synthetic data methods work?}
To tackle the problems in one-shot FL as mentioned in Section \ref{Pre}, many previous works \citep{zhang2022dense, zhang2022fine, dai2024enhancing} utilize server-side knowledge distillation to improve the global model without the need to share additional information or rely on any auxiliary dataset. They primarily focus on designing loss functions and training a generator on the server side to generate data, which is then used for knowledge transfer with ensemble models. More detailed information is provided in the Appendix.

\textbf{What are the challenges of traditional synthetic data methods?}
Although the performance of the server-side model has been improved through traditional methods, as seen from Figure \ref{fig_rethink}(b), the synthesized data distribution shows no clear boundaries. The use of such low-quality data for federated distillation limits the scope for performance enhancement. This is because (i) previous efforts mainly trained a generator with a relatively simple model architecture to synthesize data, and (ii) many generated data samples end up with mismatched labels due to the highly heterogeneous nature of federated learning data. This results in a large number of errors being learned by the global model during subsequent distillation, thereby limiting performance improvement.

Inspired by the findings of the drawbacks of traditional methods mentioned above, (i) we first consider incorporating ViTs and model inversion methods into one-shot FL to enhance the quality of generated data. As shown in Figure \ref{fig_rethink}(c), samples generated through model inversion methods exhibit more distinct boundaries in data distribution compared to traditional methods. (ii) Furthermore, due to the potential mismatch between pseudo-labels and data content, we also further process the tokens of generated data. During generation, we selectively filter out patches with higher weights, and in the subsequent distillation phase, we separately relabel patches with lower weights. Crucially, we theoretically validate that these strategies are not merely heuristic; our analysis based on algorithmic stability proves that filtering out noisy patches eliminates gradient instability, while relabeling reduces optimization variance, collectively guaranteeing a strictly tighter generalization bound. The above strategies effectively ensure that our method overcomes the limitations of traditional synthetic-data approaches.

\section{Methodology}
\label{Method}

To overcome the shortcomings of traditional synthetic data methods in one-shot FL mentioned in Section \ref{old_syn}, we propose a novel federated framework named FedMITR and the illustration of the training process in the server is demonstrated in Figure \ref{fig_framework}. After clients upload their well-trained local models to the server, we first invert well-trained networks (local models) to synthesize class-conditional images starting from random noise without using any additional information on the training dataset due to privacy concernsn in the model inversion stage. Next, we use patches with high information density, accompanied by generated pseudo-labels, to assist in training the global model, while relabeling patches with low information density for knowledge distillation. These two stages iterate multiple times on the server side.

\begin{figure*}[!t]
  \centering
  \includegraphics[width=1.0\textwidth]{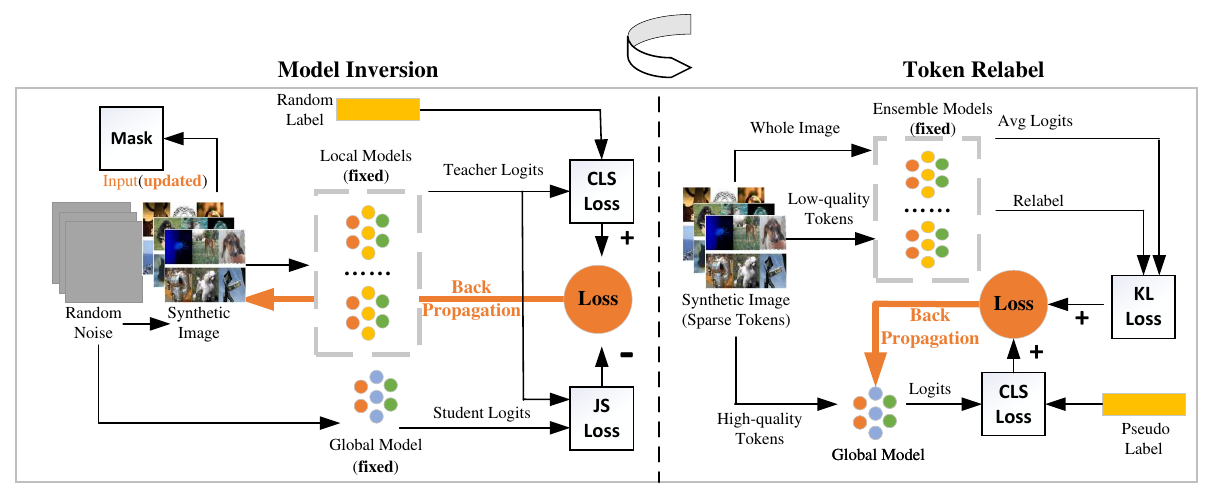}
  \caption{An illustration of server training process of FedMITR, which consists of two stages: (1) In the model inversion stage, we invert well-trained local models to synthesize input images with sparse tokens starting from random noise.  (2) In the token relabel stage , we fully utilize the role of other patch tokens encoding some information on all generated image patches.}
    \label{fig_framework}
    \vspace{-1em}
\end{figure*}

\subsection{Model Inversion}
\label{Model_Inversion}
Our goal is to invert well-trained local models to synthesize images starting from random noise without using any additional training data. Additionally, our goal is to not leak any private information from the data we generate, meaning attackers cannot predict any sensitive information of clients from the generated data. When traditional inversion methods are applied to ViT, all patches will undergo inversion. Therefore, we refer to this process as dense model inversion with redundant computation and unintended inversion of spurious correlations. 
In some models, many pieces of information in the inverted images are redundant or incorrect. In contrast, we adopt sparse model inversion, where only image patches with high-density information are inverted, while those without semantic information are filtered out through masking.

To ensure the diversity of generated data, we utilize all pretrained models from the clients for model inversion. Given the local model $\bm{f}_i(\cdot)$ parameterized by $\bm{\theta}_i$ and the server model $\bm{f}_S(\cdot)$ parameterized by $\bm{\theta}_S$, a randomly initialized input  with a new feature distribution $\hat{\bm{x}}\in\mathbb{R}^{H \times W \times C}$
(height, width, and number of channels) and  a random
uniformly sampled label $\hat{y}$. 
The model inversion process involves optimizing a classification loss, a Jensen-Shannon (JS) divergence loss with a negative scaling factor $\alpha$, and a regularization term:
\begin{align}
\label{mi_loss}
    \min_{\hat{\bm{x}}}\mathcal{L}_\text{MI}=
    \mathcal{L}_\text{CLS}\left(\bm{\theta}_i(\hat{\bm{x}}),\hat{y} \right)+ \notag\\
    \alpha \mathcal{L}_\text{JS}\left(\bm{\theta}_i(\hat{\bm{x}}),\bm{\theta}_S(\hat{\bm{x}}) \right) + 
    \mathcal{R}(\hat{\bm{x}}),
\end{align}
where $\mathcal{L}_\text{CLS}(\cdot)$ is a classification loss (e.g., cross-entropy loss) to ensure the label-conditional inversion, which desires $\hat{\bm{x}}$ could be predicted as $\hat{y}$ and exhibit discriminative features of $\hat{y}$. $\mathcal{R}(\cdot)$ is an prior image regularization term to steer $\bm{x}$ away from unrealistic images with no discernible visual information, used to penalize the total variance for local consistency \citep{dosovitskiy2016inverting}:
\begin{equation}
\label{R_loss}
   \mathcal{R}_\text{prior}(\hat{\bm{x}})=\alpha_\text{tv}\mathcal{R}_\text{tv}(\hat{\bm{x}})+\alpha_{\ell_2} \mathcal{R}_{\ell_2}(\hat{\bm{x}}),
\end{equation}
where $\mathcal{R}_\text{tv}$ and $\mathcal{R}_{\ell_2}$
are the total variance and ${\ell_2}$ norm, respectively, with scaling factors $\alpha_\text{tv}$, $\alpha_{\ell_2}$.

\looseness=-1
In Figure ~\ref{fig_framework}, the mask method refers to dividing the synthesized data generated by model inversion into two parts: tokens with high information density and tokens with low information density. The first question to address is how to identify the semantic patches crucial for inversion. 
In ViT, 
the input image $\bm{X}$ is projected to three matrices, namely query $\bm{Q}$, key $\bm{K}$, and
value $\bm{V}$ matrices. The attention operation is defined as \citep{vaswani2017attention}:
\begin{equation}
\label{attention}
   \text{Attention}(\bm{Q,K,V}) = \text{Softmax}\left(\frac{\bm{QK^T}}{\sqrt{d}}\right)\bm{V},
\end{equation}
where $d$ is the length of the query vectors in $\bm{Q}$. We define the softmax output matrix in Eq.(\ref{attention}) as the square matrix $\bm{A}$, which is known as the attention map, representing attention weights of all token pairs. We define $\bm{a}_i \triangleq \bm{A}_{[i,:]}$, $\bm{a}_i$ indicating the attention weights
from $\hat{\bm{x}}_i$ to all tokens $[\hat{\bm{x}}_\text{cls}, \hat{\bm{x}}_1,\dots, \hat{\bm{x}}_L]$.
And at iteration $t$ within the inversion process, we propose
to identify semantic patches utilizing the attention weights $\bm{a}_\text{cls}$ from the preceding iteration $t-1$. 
The output $\hat{\bm{x}}_\text{cls}$ is a linear combination of all tokens’
value vectors, weighted by $\bm{a}_\text{cls}$. Since $\hat{\bm{x}}_\text{cls}$ in the final layer serves for classification, it is rational to view $\bm{a}_\text{cls}$ as an indicator, measuring the extent to which each token contributes label-relevant information to final predictions. We first assess the importance of each remaining token based on the attention weights from the previous iteration $t-1$. Then, we stop the inversion of the mask ratio $r$ of patches with the lowest attention.

\begin{algorithm}[tb]
   \caption{Server training process of FedMITR}
   \label{alg}
\begin{algorithmic}[1]
   \STATE {\bfseries Input:} Clients' local models $\{\bm{f_1()},\cdots,\bm{f}_N()\}$, server model $\bm{f}_S()$ with parameter $\bm{\theta}_S$, synthetic dataset $\mathbb{D}_S=\emptyset$, ensemble model $\bm{E}_S$, learning rate of model inversion and token relabel $\eta_G$ and $\eta_S$, inversion iterations $T_I$, global model training epochs $T$, and batch size $b$
  \STATE {\bfseries Output:} Server model $\bm{f}_S()$ with parameter $\bm{\theta}_S$
   \FOR{epoch = $0$ to $T-1$}
   \STATE \textit{// Model Inversion}
   \FOR{$i$ = $0$ to $N-1$}
     \STATE Sample a batch of noises and labels $\{\bm{z}_i,y_i\}_{i=1}^b$
   \FOR{$t_i$ = $0$ to $T_I-1$}
      \STATE Generate $\{\hat{\bm{x}}_i\}_{i=1}^b$ with $\{\bm{z}_i\}_{i=1}^b$ 
      \STATE Update the inputs: $\hat{\bm{x}} \leftarrow\hat{\bm{x}}-\eta_G\bigtriangledown_{\hat{\bm{x}}}\mathcal{L}_\text{MI}(\hat{\bm{x}})$, where $\mathcal{L}_\text{MI}(\hat{\bm{x}})$ is defined in Eq.(\ref{mi_loss})
      \STATE Mask $\hat{\bm{x}}$ by the matrix $\bm{A}$ is defined in Section \ref{Model_Inversion}
   \ENDFOR
   \ENDFOR
   \STATE $\mathbb{D}_S\leftarrow\mathbb{D}_S\cup\{\hat{\bm{x}}_i\}_{i=1}^b$
   \STATE \textit{// Token Relabel}
   \FOR{sampling batch $\{\hat{\bm{x}}\}$ in $\mathbb{D}_S$}
   \STATE Update the server model: $\bm{\theta}_S\leftarrow\bm{\theta}_S-\eta_S\bigtriangledown_{\bm{\theta}_S}\mathcal{L}_\text{TR}(\bm{\theta}_S)$, where $\mathcal{L}_\text{TR}(\bm{\theta}_S)$ is defined in Eq.(\ref{tr_loss})
   \ENDFOR
   \ENDFOR
\end{algorithmic}
\end{algorithm}

\subsection{Token Relabel} 
Due to the heterogeneity of data in federated learning, many clients' data may not even contain certain classes in extreme cases. Therefore, many synthesized data labels mismatch with the data itself, requiring the relabeling of certain tokens. In such scenarios, the knowledge learned by each client's model is extremely limited. First, we begin by utilizing the overall image, following the approach outlined in \citep{lin2020ensemble}. However, relying solely on distillation loss is insufficient for achieving good results. Not all tokens output by ViTs can be directly and simply utilized \citep{jiang2021all} and examples include that tokens containing semantically meaningless or distractive image backgrounds do not positively contribute to the ViT predictions \citep{liang2022not}. 

Not all tokens output by Vision Transformers (ViTs) can be directly and equally leveraged for downstream tasks. Unlike CNNs that produce spatially localized features, ViTs generate a sequence of patch tokens where each token corresponds to a different image patch. However, many of these tokens may correspond to background regions or semantically irrelevant parts of the image, which carry little useful information and may even introduce noise during training. This issue is especially pronounced in federated learning settings, where data heterogeneity across clients can exacerbate the misalignment between tokens and their true semantic labels.

Previous works have shown that indiscriminately utilizing all tokens can degrade model performance, as noisy or low-information tokens dilute the learning signal and hinder convergence. Motivated by this, we propose a more refined approach to token utilization. Specifically, during the model inversion stage, we estimate the information density of each token. Tokens with high information density—typically corresponding to semantically meaningful regions of the image—are assigned pseudo-labels and directly used to train the global model. In contrast, for tokens with low information density, which are more likely to be noisy or ambiguous, we employ an ensemble model to relabel them more reliably and use them for knowledge distillation.

This selective strategy provides two key advantages. First, it improves the quality of the supervision signal by avoiding noisy labels, especially in the low-data or label-mismatched regimes typical of federated learning. Second, it allows the global model to benefit from the richer knowledge captured by the ensemble model without being misled by less informative tokens. By adapting the treatment of tokens based on their semantic content and reliability, our method strikes a better balance between label supervision and model distillation, leading to more robust global model performance.

Therefore, we utilize tokens of varying information densities generated during the model inversion stage. We train the global model with pseudo-labels for tokens with high information density and conduct distillation for tokens with low information density using an ensemble model for relabeling. The overall loss function for the entire second stage is as follows:
\begin{align}
\label{tr_loss}
    \min_{\bm{\theta}_S}
    \mathcal{L}_\text{TR}=
\mathcal{L}_\text{KD}
+ \lambda_1 \mathcal{L}_\text{CLS}\left(\bm{\theta}_S(\hat{\bm{x}}_h), \hat{y} \right) \notag\\
     + \lambda_2
    \mathcal{L}_\text{KL}(\bm{E}_S(\hat{\bm{x}}_l),\bm{f}_S(\hat{\bm{x}}_l;\bm{\theta}_S)) , 
\end{align}
where $\mathcal{L}_\text{KD}$ is defined in Eq.(\ref{kd_loss_old}) in the Appendix, $\mathcal{L}_\text{KL}(\cdot)$ is a Kullback-Leibler (KL) divergence loss. $\hat{\bm{x}}_h$ and $\hat{\bm{x}}_l$ represent the high-density and low-density information tokens respectively, and $\lambda_1, \lambda_2$ are the scaling factors. The two stages are iterated multiple times on the server-side, eventually training a suitable global model for all clients.

\subsection{Overall Training Algorithm} 
First, training is performed on each local client in FL (this paper does not focus on improvements in client-side training methods). Then, all trained local models are transmitted to the server through a single communication round. The server-side training process of FedMITR is shown in Algorithm \ref{alg}. 
After finishing the server side training process, we obtain a global model applicable to all clients.

\section{Theoretical Analysis}
\label{sec:theoretical_analysis}

In this section, to theoretically validate the effectiveness of the proposed FedMITR, we analyze its generalization performance using the framework of Algorithmic Stability \cite{bousquet2002stability,hardt2016train}. We rigorously show that our dual mechanisms—\textit{Sparse Model Inversion} and \textit{Token Relabel}—strictly reduce the generalization error bound compared to traditional Dense Inversion (DI). Due to space constraints, the complete definitions, lemmas, and proofs of the theorems are provided in Appendix \ref{sec:proof_of_model_inversion_and_token_relabel}.

\subsection{Definitions and Generalization Bounds}

We begin by formally defining the generalization error and its relationship with uniform stability. All proof details can be found in Appendix~\ref{sec:proof_of_model_inversion_and_token_relabel}.

\begin{definition}[Generalization Error]
Let $\mathcal{D}$ be the unknown data distribution and $S = \{z_1, \dots, z_N\}$ be a training set of $N$ samples drawn i.i.d. from $\mathcal{D}$. Let $A(S)$ denote the hypothesis (model parameters $\bm{\theta}$) output by a stochastic algorithm $A$ on dataset $S$. The generalization error (gap) is defined as:
\begin{equation}
    \epsilon_{gen} \triangleq \mathbb{E}_{S, A} [R(A(S)) - R_{emp}(A(S))],
\end{equation}
where $R(\bm{\theta}) = \mathbb{E}_{z \sim \mathcal{D}}[\ell(\bm{\theta}; z)]$ is the expected risk and $R_{emp}(\bm{\theta}) = \frac{1}{N} \sum_{i=1}^N \ell(\bm{\theta}; z_i)$ is the empirical risk.
\end{definition}

\begin{definition}[Uniform Stability \cite{hardt2016train}]
A stochastic algorithm $A$ is $\beta$-uniformly stable if, for all datasets $S, S'$ differing by at most one example, and for all $z$, the following inequality holds:
\begin{equation}
    \sup_{z} \mathbb{E}_A [\ell(A(S); z) - \ell(A(S'); z)] \le \beta.
\end{equation}
\end{definition}

\begin{lemma}[Generalization Bound via Stability \cite{hardt2016train}]
\label{lemma:generalization_and_stability_1}
Let algorithm $A$ be $\beta$-uniformly stable. Assume the loss function is bounded such that $0 \le \ell \le M$. Then, for any $\delta \in (0, 1)$, with probability at least $1 - \delta$, the following bound holds:
\begin{equation}
    R(A(S)) \le R_{emp}(A(S)) + 2\beta + (4N\beta + M)\sqrt{\frac{\ln(1/\delta)}{2N}}.
\end{equation}
\end{lemma}

\begin{remark}[\textbf{Relationship between $\beta$ and the $\epsilon_{gen}$}]
    By Lemma \ref{lemma:generalization_and_stability_1}, the magnitude of the generalization error bound is positively correlated with the stability parameter $\beta$. Thus, it suffices to show that a smaller $\beta$ yields a tighter $\epsilon_{gen}$. 
    The following lemma quantifies the relationship between $\beta$ and $L$.
\end{remark}

\begin{assumption}[Lipschitz Continuity and Smoothness]
\label{ass:lip_1}
For any input $\bm{x}$ and parameter $\bm{\theta}$, the loss function $\ell(\bm{\theta}; \bm{x}, y)$ satisfies:
\begin{enumerate}
    \item \textbf{$L$-Lipschitz Continuity:} The gradient norm with respect to parameters is bounded by $L$:
    $\sup_{\bm{\theta}} \| \nabla_{\bm{\theta}} \ell(\bm{\theta}) \| \le L$.
    \item \textbf{$\mu$-Smoothness:} The gradients are $\mu$-Lipschitz continuous:
    $\| \nabla \ell(\bm{\theta}_1) - \nabla \ell(\bm{\theta}_2) \| \le \mu \| \bm{\theta}_1 - \bm{\theta}_2 \|$.
\end{enumerate}
\end{assumption}

\begin{lemma}[Stability of SGD \cite{hardt2016train}]
\label{lemma:stability_of_sgd}
Suppose the loss function satisfies Assumption \ref{ass:lip_1}. Let $A$ be SGD running for $T$ steps with step sizes $\alpha_t \le c/t$. The stability parameter $\beta$ is bounded by:
\begin{equation}\label{eq:stability_bound_2}
    \beta \le \frac{1 + 1/(\mu c)}{N-1} \left( 2 c L^2 \right)^{\frac{1}{\mu c + 1}} T^{\frac{\mu c}{\mu c + 1}}.
\end{equation}
\end{lemma}

\begin{remark}[\textbf{A smaller $L$ implies a tighter generalization bound}]
    According to the relationship between stability and generalization (Lemma \ref{lemma:generalization_and_stability_1}).
    A smaller Lipschitz constant implies better stability and a tighter generalization bound. Therefore, our theoretical goal is to prove that the gradient upper bound of FedMITR is strictly smaller than that of DI, i.e., $L_{Fed} < L_{DI}$. To achieve this goal, we need to analyze the gradient norm in vision transformers.
\end{remark}

\subsection{Gradient Analysis in Vision Transformers}

To compare the Lipschitz constants $L_{Fed}$ and $L_{DI}$, we analyze the gradient dynamics within a Vision Transformer (ViT) layer from first principles.

\subsubsection{Backpropagation Dynamics}
Let the input sequence $\bm{X} \in \mathbb{R}^{N \times D}$ be partitioned into semantic foreground indices $\mathcal{I}_h$ and background noise indices $\mathcal{I}_l$. The Multi-Head Self-Attention (MSA) layer computes the output $\bm{Z}$ as:
\begin{equation*}
    \bm{Z} \!=\! \text{MSA}(\bm{X}) \!=\!\bm{A} \bm{X} \bm{W}_V,  \bm{A} \!=\! \text{Softmax}\left(\frac{\bm{X} \bm{W}_Q \bm{W}_K^T \bm{X}^T}{\sqrt{d}}\right),
\end{equation*}
where $\bm{A}$ is the attention matrix and $\bm{W}_V$ is the value projection weight. By the matrix chain rule, the gradient of the loss $\ell$ with respect to $\bm{W}_V$ is:
\begin{equation}
    \label{eq:grad_chain_1}
    \nabla_{\bm{W}_V} \ell = \bm{X}^T \bm{A}^T \bm{\delta},
\end{equation}
where $\bm{\delta} \triangleq \nabla_{\bm{Z}} \ell \in \mathbb{R}^{N \times D}$ is the error signal at the layer output. The Lipschitz constant $L$ is determined by the supremum of this gradient norm:
\begin{equation}
\label{eq:lipschitz_def_2}
    L \triangleq \sup_{\bm{\theta}, \bm{X}} \| \nabla_{\bm{\theta}} \ell \| = \sup_{\bm{\theta}, \bm{X}} \| \bm{X}^T \bm{A}^T \bm{\delta} \|.
\end{equation}
Eq. (\ref{eq:lipschitz_def_2}) reveals that instability stems from the interaction between the input energy $\|\bm{X}\|$ and the error magnitude $\|\bm{\delta}\|$.

\subsubsection{Instability of Dense Inversion}
In Dense Inversion (DI), the objective is $\mathcal{L}_{DI} \!=\! \ell_{\text{CE}}(f(\bm{X}), y)$. The error signal is $\bm{\delta}_{DI} \!=\! \text{Softmax}(f(\bm{X})) - y$. We introduce the following assumption regarding the properties of synthetic background noise.

\begin{assumption}[Statistical Orthogonality of Noise \cite{blum2020foundations}]
\label{ass:orthogonality_1}
Due to the ill-posed nature of model inversion, unconstrained background regions $\bm{X}_l$ ($j \in \mathcal{I}_l$) are filled with high-frequency isotropic noise. In high-dimensional spaces, such random vectors are asymptotically orthogonal to the fixed semantic direction of the label $y$. Thus:
\begin{equation}
    \mathbb{E}[\langle f(\bm{X}_l), y \rangle] \approx 0 \implies \mathbb{E}[\| \bm{\delta}_{DI} \|^2] \approx 1 - 1/K,
\end{equation}
where $K$ is the number of classes. This implies that fitting noise to hard labels generates a high-variance error signal.
\end{assumption}

Considering the gradient from noise tokens in DI:
\begin{equation}
    \nabla_{\bm{W}_V}^{(noise)} \mathcal{L}_{DI} = \sum_{j \in \mathcal{I}_l} \bm{x}_j^T (\bm{A}^T)_{:,j} \bm{\delta}_{DI}.
\end{equation}
The interaction between high-energy noise $\bm{x}_j$ and the large error $\bm{\delta}_{DI}$ induces a ``random walk'' in the parameter space, leading to a large Lipschitz constant $L_{DI}$.

\subsubsection{Stability of FedMITR}
FedMITR reduces the gradient norm via two mechanisms:
\begin{equation}
    \mathcal{L}_{Fed} = \underbrace{\ell_{\text{CE}}(f(\bm{X}_h), y)}_{\text{Sparsity}} + \lambda \underbrace{\ell_{\text{KL}}(f(\bm{X}_l), E(\bm{X}_l))}_{\text{Relabeling}}.
\end{equation}

\textbf{Mechanism 1: Sparsity as Gradient Truncation.}
For the sparsity term, we apply a mask $\bm{M}$ such that the effective input$\tilde{X}=M\odot X$, where $\tilde{X}_j = 0$ for noise indices $j \in \mathcal{I}_l$. Substituting this into Eq. (\ref{eq:grad_chain_1}):
\begin{equation}
    \nabla_{\bm{W}_V}^{(noise)} \mathcal{L}_{\text{Sparsity}} = \tilde{\bm{x}}_j^T \bm{A}^T \bm{\delta} = \mathbf{0}^T \bm{A}^T \bm{\delta} = 0.
\end{equation}
Sparsity physically eliminates the input factor in the gradient chain, mathematically truncating the high-variance interaction derived in DI.

\textbf{Mechanism 2: Relabeling as Variance Reduction.}
For the relabeling term, the error is $\bm{\delta}_{Fed} = p(\bm{X}_l) - q(E(\bm{X}_l))$. We invoke the statistical property of distillation to bound this error.

\begin{lemma}[Gradient Variance Reduction \cite{menon2021statistical}]
\label{lemma:menon_variance_1}
The expected squared norm of the stochastic gradient decomposes into the norm of the mean gradient and the trace of the covariance (variance). For noise inputs, the mean gradient vanishes. The variance of gradients from soft labels (distillation) is strictly bounded by the student-teacher discrepancy, which is significantly lower than the variance from hard labels:
\begin{equation}
    \mathbb{E}[\|\bm{\delta}_{soft}\|^2] \ll \mathbb{E}[\|\bm{\delta}_{hard}\|^2].
\end{equation}
\end{lemma}

\subsection{Main Theorem}

\begin{theorem}[Lipschitz Constant Reduction]
\label{thm:lipschitz_reduction_1}
Let $L_{DI}$ and $L_{Fed}$ be the Lipschitz constants of the loss gradients with respect to model parameters for Dense Inversion and FedMITR, respectively. Under Assumptions \ref{ass:lip}, \ref{ass:orthogonality_1} and Lemma \ref{lemma:menon_variance_1}, we have:
\begin{equation*}
    L_{Fed} < L_{DI}.
\end{equation*}
\end{theorem}

\begin{proof}
The total Lipschitz constant $L$ is the supremum of the gradient norm, composed of foreground and background contributions.
For Dense Inversion:
\begin{equation*}
    L_{DI} \approx L_{fg} + \sup (\| \bm{X}_{noise} \| \cdot \| \bm{A} \| \cdot \| \bm{\delta}_{hard} \|).
\end{equation*}
For FedMITR:
\begin{equation*}
    L_{Fed} \approx L_{fg} + \underbrace{0}_{\text{Sparsity}} + \underbrace{\lambda \sup (\| \bm{X}_{noise} \| \cdot \| \bm{A} \| \cdot \| \bm{\delta}_{soft} \|)}_{\text{Relabeling}}.
\end{equation*}
The sparsity mechanism strictly zeroes out the noise-induced gradient term. The relabeling mechanism introduces a term with $\bm{\delta}_{soft}$. From Lemma \ref{lemma:menon_variance_1}, we know $\mathbb{E}[\|\bm{\delta}_{soft}\|^2] \ll \mathbb{E}[\|\bm{\delta}_{hard}\|^2]$ due to the variance reduction property of distillation and the consistency between student and teacher on the noise manifold.
Therefore, the upper bound of the gradient norm in FedMITR is strictly smaller than in Dense Inversion, i.e., $L_{Fed} < L_{DI}$.
\end{proof}


\begin{remark}[\textbf{FedMITR attains a tighter generalization bound}]
    Substituting the result of Theorem \ref{thm:lipschitz_reduction_1} into Lemma \ref{lemma:stability_of_sgd}, we obtain $\beta_{Fed} < \beta_{DI}$. Based on the relationship between stability and the generalization error bound (Lemma \ref{lemma:generalization_and_stability_1}), we obtain that FedMITR can significantly tighten the generalization bound. This theoretically confirms that FedMITR effectively stabilizes the non-convex optimization of ViTs by eliminating spurious noise gradients and enforcing consistent teacher-student dynamics. This result aligns with our empirical observations in Section \ref{exp}, where FedMITR consistently outperforms baselines on heterogeneous data.
\end{remark}

\section{Experiments} \label{exp}

In this section, we conduct extensive experiments to verify the effectiveness of our proposed approach. 

\subsection{Experimental Setup}

\begin{figure*}[!t]
  \centering
  \includegraphics[width=0.9\textwidth]{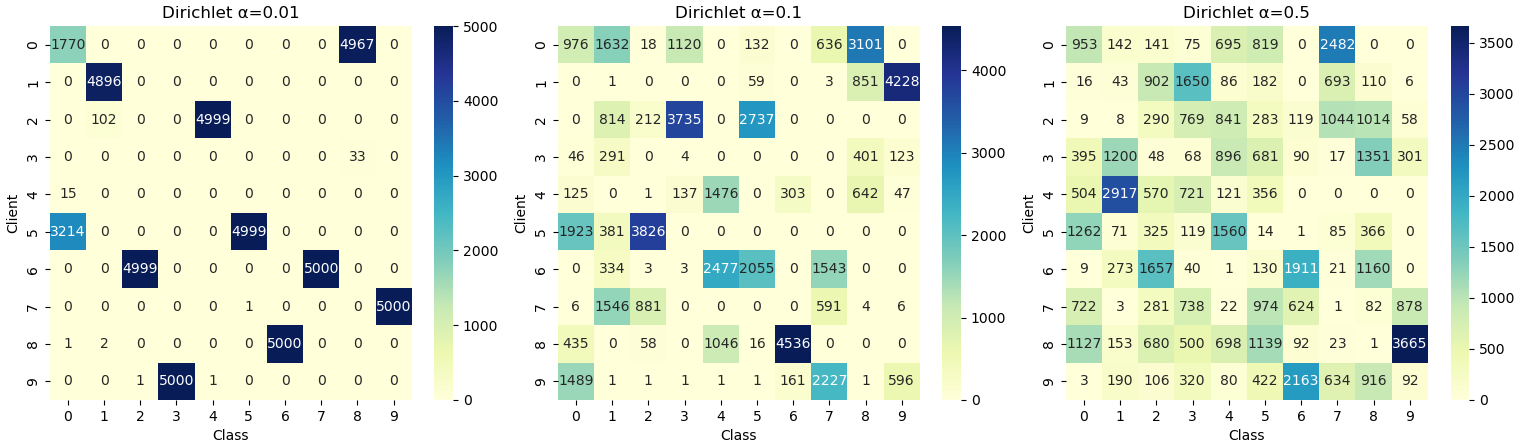}
  \caption{Label Distribution per Client under Dirichlet Partition.}
    \label{fig3}
    \vspace{-1em}
\end{figure*}

To comprehensively evaluate the effectiveness of FedMITR, we conduct experiments on four diverse real-world datasets: CIFAR10, CIFAR100, OfficeHome, and Mini-ImageNet. To simulate realistic non-IID federated learning scenarios, we adopt two data partitioning strategies: the Dirichlet distribution partition ($\bm{p}_k\sim Dir(\alpha)$) and the pathological partition based on class quantities ($\#C = k$). We benchmark our method against five representative baselines, including the classical FedAvg \citep{mcmahan2017communication}, state-of-the-art one-shot FL methods (FedFTG \citep{zhang2022fine}, DENSE \citep{zhang2022dense}, Co-Boosting \citep{dai2024enhancing}), and the inversion-based method DeepInversion \citep{Yin_2020_CVPR}. All experiments are conducted with 10 clients using pre-trained DeiT/16-Tiny models as the backbone. Due to page limitations, detailed descriptions of datasets, baseline implementations, and hyperparameter configurations are provided in \textbf{Appendix~\ref{sec:appendix_exp_details}}.

\subsection{General Results and Analysis}
\begin{table*}[!t]
\caption{Test accuracy of the server model of different methods on three datasets and across five levels of statistical heterogeneity (lower $\alpha$ is more heterogeneous).} 
\label{tab_main}
\centering
\begin{adjustbox}{width=0.9\textwidth}
\begin{tabular}{@{}c|c|cccccc@{}}
\toprule
Dataset    
& $\alpha$ & FedAvg & FedFTG & DENSE & Co-Boosting & DeepInversion & FedMITR  \\ \midrule
\multirow{5}{*}{CIFAR10} 
& 0.01 & 11.70$\pm$1.87 & 12.73$\pm$2.91 & 12.05$\pm$2.30 & 12.07$\pm$2.23 & 13.42$\pm$1.84 & \textbf{19.19$\pm$2.33} \\
& 0.05 & 15.87$\pm$2.74 & 16.16$\pm$2.44 & 16.41$\pm$2.36 & 17.31$\pm$2.65 & 22.01$\pm$2.44 & \textbf{26.78$\pm$2.12} \\
& 0.1 & 24.30$\pm$3.72 & 25.05$\pm$4.28 & 25.19$\pm$3.12 & 26.88$\pm$2.74 & 33.77$\pm$2.03 & \textbf{36.97$\pm$2.98} \\
& 0.3 &37.93$\pm$3.73 & 40.01$\pm$5.35 & 39.02$\pm$4.05 & 40.57$\pm$5.57 & 44.63$\pm$3.70 & \textbf{51.69$\pm$2.86} \\ 
& 0.5 & 39.37$\pm$1.53 & 42.02$\pm$2.81 & 40.08$\pm$1.79 & 41.55$\pm$2.17 & 45.00$\pm$2.11 & \textbf{49.45$\pm$5.86} \\ 
\midrule

\multirow{5}{*}{OfficeHome} 
& 0.01 & 8.06$\pm$1.40 & 8.79$\pm$1.49 & 8.62$\pm$1.40 & 8.82$\pm$1.35 & 11.88$\pm$1.51 & \textbf{24.05$\pm$1.19} \\
& 0.05 & 13.68$\pm$1.66 & 14.61$\pm$1.46 & 14.33$\pm$1.42 & 14.70$\pm$1.38 & 18.72$\pm$0.95 & \textbf{30.04$\pm$1.47} \\
& 0.1 & 17.63$\pm$2.39 & 19.10$\pm$1.94 & 18.49$\pm$2.19 & 18.82$\pm$2.25 & 23.89$\pm$1.41 & \textbf{32.81$\pm$1.55} \\
& 0.3 & 26.88$\pm$1.56 & 28.79$\pm$1.27 & 28.24$\pm$1.47 & 28.60$\pm$1.10 & 33.54$\pm$1.68 & \textbf{35.53$\pm$0.32} \\ 
& 0.5 & 31.13$\pm$2.63 & 32.86$\pm$2.88 & 32.48$\pm$2.77 & 32.91$\pm$2.72 & 37.87$\pm$3.25 & \textbf{38.15$\pm$1.63} \\ 
\midrule

\multirow{5}{*}{\begin{tabular}[c]{@{}c@{}}Mini-\\ ImageNet\end{tabular}}
& 0.01 & 13.99$\pm$1.83 & 14.73$\pm$2.02 & 14.49$\pm$1.84 & 15.08$\pm$1.90 & 22.49$\pm$1.89 & \textbf{45.26$\pm$5.14} \\
& 0.05 & 37.98$\pm$2.16 & 38.92$\pm$2.00 & 38.65$\pm$1.95 & 39.03$\pm$2.39 & 47.34$\pm$2.89 & \textbf{62.15$\pm$0.28} \\
& 0.1 & 52.48$\pm$2.10 & 53.62$\pm$1.79 & 53.19$\pm$1.87 & 53.47$\pm$2.05 & 60.28$\pm$2.10 &\textbf{68.21$\pm$2.01} \\
& 0.3 & 76.84$\pm$1.71 & 77.64$\pm$1.46 & 77.32$\pm$1.71 & 77.46$\pm$1.66 & \textbf{79.06$\pm$1.31} & 77.44$\pm$1.50 \\
& 0.5 & 81.87$\pm$0.23 & 82.85$\pm$0.18 & 82.16$\pm$0.08 & 82.21$\pm$0.19 & \textbf{83.09$\pm$0.77} & 82.18$\pm$0.22 \\ 
\bottomrule
\end{tabular}
\end{adjustbox}
\vspace{-1em}
\end{table*}
\textbf{Overall Comparison.} To evaluate the effectiveness of our method, we conduct experiments under various non-IID settings by varying $\alpha=\{0.01,0.05,0.1,0.3,0.5\}$ and report the performance across different datasets and methods in Table \ref{tab_main}. From the table, we can conclude that FedMITR consistently outperforms all other baselines in all settings, especially in highly heterogeneous scenarios where the Dirichlet distribution parameter is very small. Notably, in many settings, FedMITR achieves over a significant accuracy improvement compared to the best baseline, DeepInversion. Our approach shows a more significant improvement compared to traditional methods, as these methods do not use ViTs for model training; instead, most methods use CNNs to guide the training of the generator. However, when the heterogeneity is low, such as $\alpha=0.5$ and $\alpha=0.3$ on Mini-Imagenet, the accuracy of FedMITR is lower than DeepInversion. This is because when the heterogeneity is low, local models are already well-trained and do not require additional relabeling to facilitate federated distillation. In conclusion, the superiority of our proposed method can be attributed to utilizing the ViT model and model inversion to use all patches, which achieves better utilization of synthesized data.

\textbf{Extension to Extreme Heterogeneity.}
In Table \ref{tab_extreme}, we use 10 clients, with each client assigned 1 (extreme heterogeneity) and 3 labels in the CIFAR10 dataset with a total of 10 categories, and 10 labels in the Mini-ImageNet dataset with a total of 100 categories. In such cases, the accuracy of traditional methods is very low and we can conclude that in settings of extreme heterogeneity, FedMITR can achieve larger improvements compared to traditional methods.

\begin{table*}[!t]
\caption{Test accuracy of the server model on two datasets under extreme heterogeneity setting.}
\label{tab_extreme}
\small
\centering
\begin{adjustbox}{width=0.9\textwidth}
\begin{tabular}{@{}c|c|cccccc@{}}
\toprule
Dataset 
& Partition & FedAvg & FedFTG & DENSE & Co-Boosting & DeepInversion & FedMITR \\ \midrule
\multirow{2}{*}{CIFAR10}
& \#C=1   & 9.10$\pm$1.49 & 9.59$\pm$2.06 & 9.28$\pm$1.60 & 9.37$\pm$1.76 & 9.77$\pm$1.85 & \textbf{13.98$\pm$3.46} \\
& \#C=3  & 20.12$\pm$4.28 & 21.88$\pm$1.51 & 20.99$\pm$4.22 &21.26$\pm$4.59 & 28.06$\pm$4.91 & \textbf{33.85$\pm$2.48} \\
\midrule
\multirow{1}{*}{\begin{tabular}[c]{@{}c@{}}Mini-ImageNet\end{tabular}} 
& \#C=10  & 7.80$\pm$0.44 & 8.55$\pm$0.90 & 8.28$\pm$0.66 & 8.39$\pm$0.75 & 13.59$\pm$0.99 & \textbf{21.81$\pm$1.22} \\
\bottomrule
\end{tabular}
\end{adjustbox}
\vspace{-1em}
\end{table*}

\textbf{Effects of the proposed components.} 
To further assess the effectiveness of FedMITR, which involves model inversion and token relabel, we conduct experiments in different models across four datasets under $Dir(0.1)$ heterogeneous client model setting. And we further study the effectiveness of our proposed components. Table \ref{tab_ablation_dir0.1} displays four methods: the baseline FedAvg, knowledge distillation based only on model inversion (inversion + KD), training based only on pseudo-labels using model inversion (inversion + PL), and our proposed method FedMITR. The results in the table indicate that, upon obtaining data from model inversion, individually performing knowledge distillation or training with pseudo-labels can enhance the final server model's performance. Moreover, our approach, which combines both strategies and further conducts relabeling followed by distillation on low-information-density tokens, achieves the best performance.

\begin{table}[h]
\caption{Test accuracy of server model in different models across four datasets under $Dir(0.1)$ heterogeneous client model setting.}
\label{tab_ablation_dir0.1}
\small
\centering
\begin{adjustbox}{width=0.5\textwidth}
\begin{tabular}{@{}c|c|cccccc@{}}
\toprule
Model & Method & CIFAR10 & CIFAR100 & OfficeHome & Mini-ImageNet  \\ \midrule
\multirow{4}{*}{DeiT/16-Tiny} & FedAvg & 24.98 & 9.33 & 20.39 & 50.92  \\
 &Inversion + KD & 33.25 & 11.08 & 25.44 & 59.19  \\
 & Inversion + PL & 34.72 & 11.86 & 34.41 & 66.07  \\
  & FedMITR & \textbf{35.69} & \textbf{13.84} & \textbf{34.60} & \textbf{68.67}  \\
\midrule
 \multirow{4}{*}{DeiT/16-Base} & FedAvg & 57.44 & 30.62 & 39.40 & 80.19 \\
 & Inversion + KD & 58.92 & 32.28 & 40.46 & 80.47 \\
 & Inversion + PL & 65.93 & 34.61 & 45.98 & \textbf{88.52}  \\
 & FedMITR & \textbf{66.54} & \textbf{35.19} & \textbf{46.70} & 88.37  \\
\midrule
\multirow{4}{*}{ViT/16-Small} & FedAvg & 50.98 & 12.29 & 37.91 & 78.37  \\
 & Inversion + KD & 52.48 & 14.67 & 40.46 & 79.15  \\
 & Inversion + PL & 53.50 & 15.87 & 42.24 & 86.18 \\
 & FedMITR & \textbf{53.82} & \textbf{17.66} & \textbf{46.07} & \textbf{87.88}  \\
 \bottomrule
\end{tabular}
\end{adjustbox}
\end{table}

\textbf{More details about the ablation experiments.}
In our method, the components of the loss function in Eq.(\ref{tr_loss}) collectively form the overall tokens being processed. Therefore, instead of directly removing some components while retaining others for ablation experiments, we employed alternative methods for experimental validation in Table \ref{tab_ablation_dir0.1}. Here, Inversion + KD refers to knowledge distillation based only on model inversion, while Inversion + PL refers to knowledge distillation using only pseudo-labels without employing token relabel. Below is our additional analysis of the hyperparameters. In the face of the highly heterogeneous challenges in Federated Learning, many generated data labels do not match the data. This results in a large number of errors being learned by the global model during subsequent train, thereby limiting performance improvement. So an excessively high $\lambda_1$ will lead to a performance drop. The loss weighted by parameter $\lambda_2$ represents the tokens involved in knowledge distillation through ensemble model re-labeling. Since it is not influenced by potentially erroneous pseudo-labels, it maintains higher robustness and is less sensitive to parameter variations. The Table \ref{tab_1} is the additional ablation experiment about Eq.(\ref{tr_loss}).

\begin{table}[!t]
\caption{Ablations on different components of our method in Mini-ImageNet in three random seeds and across three levels of statistical heterogeneity.} 
\label{tab_1}
\centering
\small
\begin{adjustbox}{width=0.5\textwidth}
\begin{tabular}{@{}c|c|cccc@{}}
\toprule
Dataset    
& $\alpha$ & w/ $\mathcal{L}_{\text{KD}}$ & w/o $\mathcal{L}_{\text{CLS}}$ & w/o $\mathcal{L}_{\text{KL}}$ & FedMITR  \\ \midrule
\multirow{3}{*}{\begin{tabular}[c]{@{}c@{}}Mini-\\ ImageNet\end{tabular}} 
& 0.01 & 22.49$\pm$1.89  & 37.34$\pm$1.22 & 42.04$\pm$1.49 & \textbf{45.26$\pm$5.14} \\
& 0.05 & 47.34$\pm$2.89  & 51.24$\pm$1.69 & 57.78$\pm$1.00 & \textbf{62.15$\pm$0.28} \\
& 0.1 & 60.28$\pm$2.10 & 61.44$\pm$2.47 & 63.12$\pm$3.59 & \textbf{68.21$\pm$2.01} \\
\bottomrule
\end{tabular}
\end{adjustbox}
\vspace{-1em}
\end{table}

\textbf{Comparison in Terms of Communication Overhead:}
FedMITR is compared with classical multi-round communication methods in federated learning such as FedAvg, FedProx, and SCAFFOLD. Metrics including communication rounds, data per round (MB), total communication volume (MB), whether it is one-shot communication, and accuracy are listed in Table~\ref{tab_3_5}. Experiments are conducted on the heterogeneous Mini-ImageNet dataset with a Dirichlet distribution parameter of 0.1. The number of clients is 10. For multi-round communication methods, 50\% of local models are randomly selected for transmission in each round. A pre-trained DeiT-Tiny model of 21.8MB is used. As shown in the table, under comparable performance, the method proposed in this chapter can significantly reduce communication overhead in a one-shot federated learning framework. By shifting the core training process to the server side, it reduces the multi-round communication cost between clients and the server.

\begin{table}[ht]
\caption{Comparison of Communication Overhead between FedMITR and Traditional Federated Learning Methods}
\label{tab_3_5}
\small
\centering
\begin{adjustbox}{width=0.5\textwidth}
\begin{tabular}{@{}c|c|c|c|c|c@{}}
\toprule
Method & Communication Rounds & Data per Round(MB) & Total Communication(MB)  & One-Shot & Accuracy(\%)  
\\ \midrule
FedAvg & 50 & 109 & 5450 & \ding{55} & 67.13 \\
FedProx & 50 & 109 & 5450 & \ding{55} & 69.24 \\
SCAFFOLD & 50 & 109 & 5450 & \ding{55}&  69.78\\
FedMITR & 1 & 218 & 218 & $\checkmark$ & 68.21 \\ \bottomrule
\end{tabular}
\end{adjustbox}
\end{table}

\textbf{Comparison in Terms of Communication Efficiency:}
As illustrated in Fig.~\ref{fig4}, the proposed FedMITR method achieves a competitive accuracy of 68.21\% with only one round of communication, outperforming the classical FedAvg approach in terms of communication efficiency. In contrast, FedAvg requires approximately 100 communication rounds to slightly surpass the accuracy of FedMITR. This demonstrates that FedMITR can effectively reduce the number of communication rounds while maintaining high model performance. The rapid convergence of FedAvg shown in the figure further highlights that FedMITR’s one-shot training can match or even exceed multi-round training performance.
\begin{figure}
\begin{center}
\includegraphics[width=0.5\textwidth]{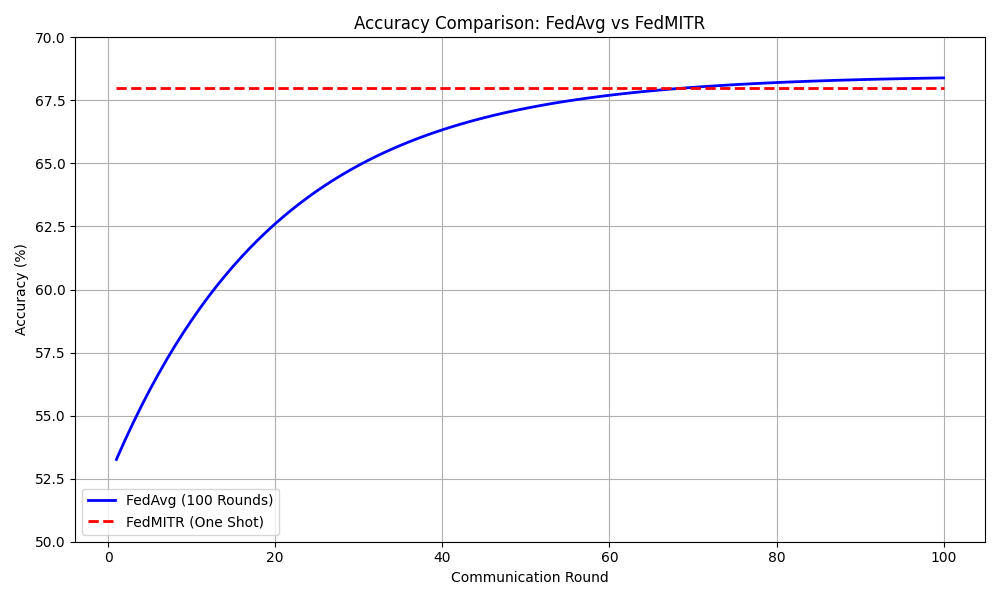}
\caption{Comparison of Accuracy Between the FedAvg Method with 100 Communication Rounds and the One Shot FedMITR Method in Federated Learning.}
\label{fig4}
\end{center}
\vspace{-2em}
\end{figure}

\textbf{Different Number of Clients.} We also evaluate the performance of these methods by varying the number of clients participating $N=\{5,10,20,50\}$ in one-shot FL in Table \ref{tab_client}. From the table, Although there is a slight decrease in overall performance when increasing the number of clients, FedMITR still achieves the best performance, reaffirming the effectiveness of our approach.
\vspace{-1mm}
\begin{table}[H]
\caption{Test accuracy of the server model in Mini-ImageNet across different numbers of clients under $Dir(0.1)$ heterogeneous client model setting.}
\label{tab_client}
\small
\centering
\begin{adjustbox}{width=0.5\textwidth}
\begin{tabular}{@{}c|cccccc@{}}
\toprule
$N$ & FedAvg & FedFTG & DENSE & Co-Boosting & DeepInversion & FedMITR \\ \midrule
5   & 65.31 & 65.80 & 65.82 & 66.18 & 68.05 & \textbf{71.31} \\
10  & 54.87 & 55.69 & 55.33 & 55.83 & 62.71 & \textbf{69.94} \\
20  & 36.15 & 37.67 & 37.35 & 37.90 & 44.19 & \textbf{54.80} \\
50  & 24.54 & 25.24 & 25.11 & 25.38 & 30.40 & \textbf{49.33} \\ \bottomrule
\end{tabular}
\end{adjustbox}
\end{table}

\textbf{Hyperparameters senstivity.}
To measure the influence of hyperparameters, we select
$\lambda_1$ and $\lambda_2$ from $\{0.25, 0.50, 0.75, 1.00\}$ and select the
mask ratio $r$ in $\{0.1, 0.2, 0.3, 0.4\}$. Figure \ref{fig5} illustrates the test accuracy in term of the box plot, where (a) suggests that an excessively high $\lambda_1$ will lead to a performance drop. This is because too many pseudo-labels are undesirable in heterogeneous conditions as many pseudo-labels do not match the synthetic data. (b) indicates that the parameter $\lambda_2$ is not sensitive, while (c) shows that the mask ratio $r$ should not be too high either. 

\begin{figure*}[htbp]
  \centering
  \includegraphics[width=0.9\textwidth]{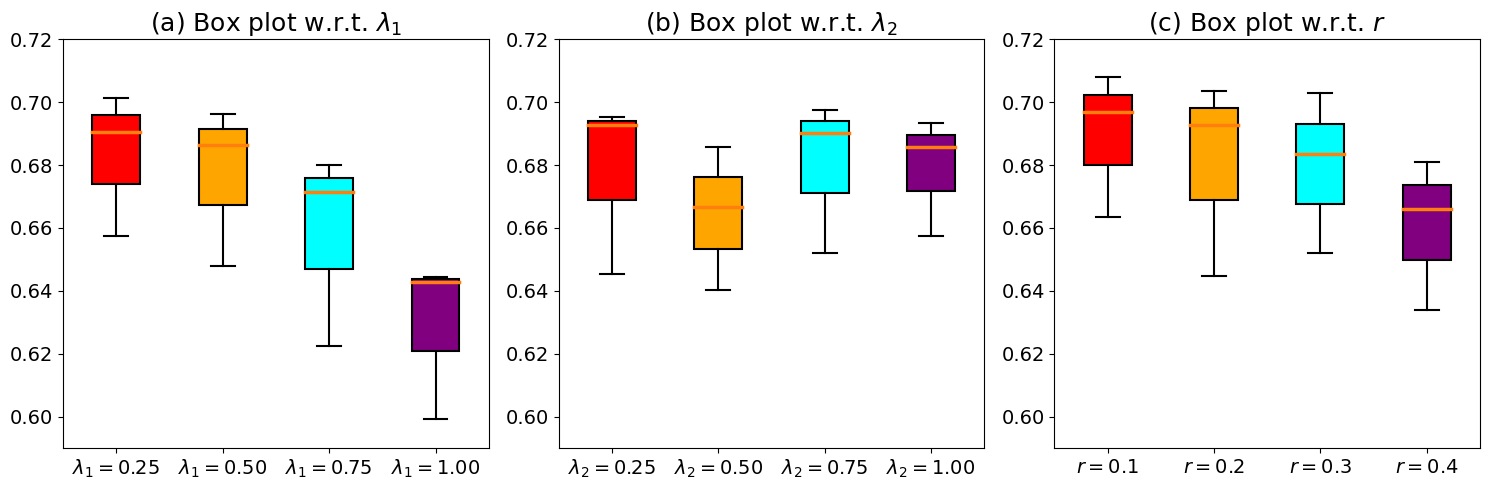}
  \caption{Test accuracy of the server model of FedMITR using different hyper-parameters (a) $\lambda_1$, (b) $\lambda_2$, (c) the mask ratio $r$ on Mini-ImageNet under $Dir(0.1)$ heterogeneous client model setting.}
\label{fig5}
\vspace{-1em}
\end{figure*}

\textbf{Visualization of synthetic data.} To compare the synthetic data (including tokens with high information density and low information density) in our method with the training data, we visualize the synthetic data on the OfficeHome and Mini-ImageNet datasets in Fig \ref{fig6}. 
As shown in the figure, the first/fourth row represents the original data of the OfficeHome/Mini-ImageNet dataset, while the rest are synthetic data generated by models trained on the these datasets. Among them, the second/fifth row consists of selected high-information-density patches, while the third/last row consists of low-information-density patches. 
Visually, we can not obtain any privacy information because the synthetic images are dissimilar to the original images, effectively reducing the probability of leaking sensitive client information. 
It's worth noting that although the synthetic data appear visually distinct from the real data, our method still achieves higher performance than other baseline methods by training with these synthetic data.
\begin{figure}
\begin{center}
\includegraphics[width=0.4\textwidth]{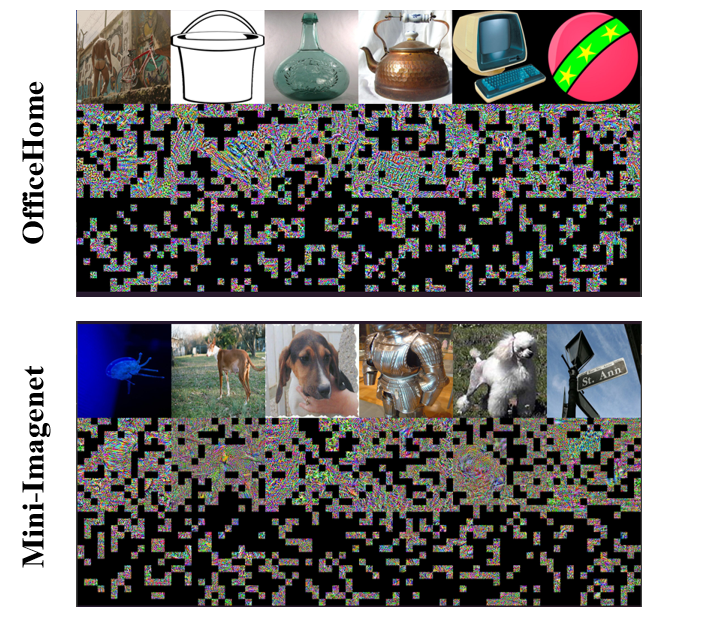}
\caption{Visualization of synthetic data and training data}
\label{fig6}
\end{center}
\vspace{-2em}
\end{figure}

\section{Conclusion}
In this paper, we first present a comprehensive critique of existing methods using synthetic data in federated learning, emphasizing their drawbacks and limitations. Then, we propose a novel Federated Model Inversion and Token Relabel framework named FedMITR. This framework generates synthetic images through ViTs and efficiently utilizes all tokens of the generated images to train the global model. In addition, we further provide a religious theoretical stability analysis to evaluate the effectiveness of the proposed approach.
Extensive analytical and empirical studies on various datasets verify the effectiveness of our method, consistently outperforming other baseline methods under diverse heterogeneous settings.


\bibliography{ref}
\bibliographystyle{IEEEtran}

\newpage
\onecolumn
\appendix
\section{Appendix}

\subsection{More details about the Experiment}
\label{sec:appendix_exp_details}
\textbf{Datasets and partitions.} 
Our experiments are conducted on the following four popular real-world datasets: CIFAR10 \citep{krizhevsky2009learning}, CIFAR100 \citep{krizhevsky2009learning}, OfficeHome \citep{venkateswara2017deep} and Mini-ImageNet \citep{vinyals2016matching}. CIFAR10 \citep{krizhevsky2009learning} consists of 60,000 color images in 10 classes, 50,000 for train, and 10,000 for test, while CIFAR100 \citep{krizhevsky2009learning} dataset is similar to the CIFAR10 dataset but it has 100 classes containing 600 images each. 
OfficeHome \citep{venkateswara2017deep} is a image recognition dataset that includes 15,588 images of 65 classes from four different domains (art, clipart, product, and real-world). Mini-ImageNet \citep{vinyals2016matching} is a small subset extracted from the ImageNet-1K \citep{russakovsky2015imagenet} dataset, consisting of 100 categories with 600 images per category, totaling 60,000 images. We split
80\% data as the training set and 20\% of that as the testing set. Each image is resized to a 224×224 color image. All available test data is used to evaluate the final server model. 
To simulate real-world applications, we adopt two different kinds of partition: 1) $\bm{p}_k\sim Dir(\alpha)$: for each class, we allocate a $\bm{p}^i_k$ proportion of the data of class $i$ to client $k$. The parameter $\alpha$ controls the level of statistical imbalance, with a smaller $\alpha$ inducing more skewed label distributions among local clients, as shown in  Fig \ref{fig3}. 2) $\#C = k$: each client only has data from $k$ classes and we assign $k$ random classes for each client.

\textbf{Baselines.}
To ensure a fair comparison, we disregarded methods that require downloading auxiliary models or additional datasets. Furthermore, due to the single round of communication, regularization-based aggregation methods or similar approaches that rely on multiple iterations are ineffective. 
Therefore, against four existing FL methods: FedAvg\citep{mcmahan2017communication}, FedFTG \citep{zhang2022fine}, DENSE \citep{zhang2022dense} and Co-Boosting \citep{dai2024enhancing}. FedAvg\citep{mcmahan2017communication} learns a shared model by aggregating locally-computed updates and iteratively updating through multiple rounds of communication between clients and the server. FedFTG explores the input space of local models through a generator and uses it to transfer knowledge from the local models to the global model. Additionally, FedFTG proposes a hard sample mining scheme to achieve effective knowledge distillation throughout the training process. Furthermore, FedFTG also develops customized label sampling and class-level ensemble techniques to maximize knowledge utilization, which implicitly mitigates distribution differences among clients. The two methods mentioned above are based on traditional multi-round communication federated learning, so they are set to communicate only once in this paper.
DENSE \citep{zhang2022dense} train a generator that considers similarity, stability, and transferability and performe federated distillation on the server side. Co-Boosting \citep{dai2024enhancing} uses the current Ensemble to synthesize higher-quality samples in an adversarial manner. These hard samples are then employed to promote the quality of the Ensemble by adjusting the ensembling weights for each client model. 
Since FedMITR is a DFKD method, we also use the data-free method DeepInversion \citep{Yin_2020_CVPR} in model inversion as a comparison method, applying it to the one-shot FL setting. The DeepInversion \citep{Yin_2020_CVPR} optimizes the input data while keeping the teacher model fixed during data synthesis, and it regularizes the distribution of intermediate feature maps using the information stored in the teacher's batch normalization layers. Additionally, adaptive deep inversion is employed to enhance the diversity of the synthesized images, thereby maximizing the Jensen-Shannon divergence between the logits of the teacher and student networks.

\textbf{Configurations.} Unless otherwise stated, we conduct experiments with 10 clients. All models are accessible from timm. For each client's training, we use pre-trained DeiT/16-Tiny on ImageNet-1K \citep{russakovsky2015imagenet} as train models and use the SGD optimizer with learning rate=0.001, momentum=0.9 and weight decay=1e-4. We set the batch size to 64 and the local epoch to 50. We perform 100 iterations for model inversion using the Adam optimizer with a learning rate $\eta_G=0.001$ and $(\beta_1,\beta_2)=(0.5,0.99)$ about each local model. The image regularization term scaling factor $\alpha_\text{tv}$ is set as 1e-4 and the mask ratio $r$ is set as 0.3. The scaling factors $\lambda_1$ and $\lambda_2$ are set to 0.5. The batch size of synthetic data is set to 64. For the training of the global model, we use the SGD optimizer with a learning rate $\eta_S=0.001$, momentum=0.9 and weight decay=1e-4. The factor $\alpha$ of JS divergence loss is set to 1.0. The framework is implemented with PyTorch and is trained on a single NVIDIA RTX 3090 GPU.

\textbf{More details about the experiment results.}
Due to space limitations and the relatively poor performance of the DeiT-Tiny model on CIFAR-100 (as this paper focuses on server-side improvements and does not use more advanced training methods during local training), we does not include CIFAR-100 results in the main table. Due to overall low performance, we did not conduct in-depth research on other aspects of the experiments. However, our method, FedMITR, is still capable of improving model performance in environments with high heterogeneity on a relative scale in Table \ref{tab_2}. 

\begin{table}[h]
\caption{Test accuracy of the server model of different methods in CIFAR100 in three random seeds and across three levels of statistical heterogeneity.} 
\label{tab_2}
\centering
\small
\begin{tabular}{@{}c|c|cccc@{}}
\toprule
Dataset    
& $\alpha$ & FedAvg & Co-Boosting & DeepInversion & FedMITR  \\ \midrule
\multirow{5}{*}{CIFAR100} 
& 0.01 & 4.02$\pm$0.66  & 4.51$\pm$0.57 & 5.89$\pm$0.72 & \textbf{8.89$\pm$1.23} \\
& 0.05 & 6.62$\pm$0.78  & 7.33$\pm$0.74 & 9.02$\pm$0.67 & \textbf{11.35$\pm$0.89} \\
& 0.1 & 8.55$\pm$1.12 & 9.23$\pm$1.07 & 10.86$\pm$1.60 & \textbf{13.12$\pm$1.03} \\
& 0.3 & 12.30$\pm$0.31  & 13.25$\pm$0.75 & 15.24$\pm$1.07 & \textbf{16.45$\pm$1.92} \\ 
& 0.5 & 13.99$\pm$0.82 & 14.87$\pm$1.11 & \textbf{17.12$\pm$1.65} & 17.09$\pm$1.90 \\ 
\bottomrule
\end{tabular}
\end{table}

To evaluate the effectiveness of our method, we conduct experiments under various non-IID settings by varying $\alpha=\{0.01,0.05,0.1,0.3,0.5\}$ and $\#C=k$ in Tables \ref{tab_3} to \ref{tab_7}. Below are the detailed results for each random seed.

\begin{table}[H]
\caption{Test accuracy of the server model of different methods in CIFAR10 in three random seeds under extreme heterogeneity setting.} 
\label{tab_3}
\centering
\small
\begin{tabular}{@{}c|c|cccccc@{}}
\toprule
$\#C=k$ 
& seed & FedAvg & FedFTG & DENSE & Co-Boosting & DeepInversion & FedMITR  \\
\midrule
\multirow{3}{*}{$\#C=1$ } 
& 0 & 7.97 & 7.95 & 7.99 & 7.98 & 8.27 & \textbf{11.15} \\
& 1 & 10.78 & 11.90 & 11.07 & 11.34 & 11.83 & \textbf{12.95} \\
& 2 & 8.54 & 8.91 & 8.79 & 8.78 & 9.20 & \textbf{17.83} \\
\midrule
\multirow{3}{*}{$\#C=3$} 
& 0 & 15.19 & 20.14 & 16.11 & 15.99 & 22.95 & \textbf{31.89} \\
& 1 & 22.27 & 22.64 & 23.42 & 24.38 & 32.74 & \textbf{33.02} \\
& 2 & 22.90 & 22.85 & 23.43 & 23.41 & 28.48 & \textbf{36.63} \\
\bottomrule
\end{tabular}
\end{table}

\begin{table}[H]
\caption{Test accuracy of the server model of different methods in CIFAR10  in three random seeds and across five levels of statistical heterogeneity (lower $\alpha$ is more heterogeneous).} 
\label{tab_4}
\centering
\small
\begin{tabular}{@{}c|c|cccccc@{}}
\toprule
$p \sim Dir(\alpha)$  
& seed & FedAvg & FedFTG & DENSE & Co-Boosting & DeepInversion & FedMITR  \\
\midrule
\multirow{3}{*}{0.01} 
& 0 & 10.32 & 11.02 & 10.33 & 10.37 & 12.77 & \textbf{20.49} \\
& 1 & 13.83 & 16.09 & 14.66 & 14.59 & 15.49 & \textbf{20.57} \\
& 2 & 10.94 & 11.09 & 11.15 & 11.25 & 11.99 & \textbf{16.50} \\
\midrule
\multirow{3}{*}{0.05} 
& 0 & 13.85 & 14.75 & 14.69 & 14.74 & 19.59 & \textbf{30.60} \\
& 1 & 18.99 & 18.98 & 19.10 & 20.03 & 24.47 & \textbf{31.68} \\
& 2 & 14.77 & 14.76 & 15.45 & 17.16 & 21.96 & \textbf{26.54} \\
\midrule
\multirow{3}{*}{0.1} 
& 0 & 27.63 & 28.59 & 28.19 & 29.71 & 36.01 & \textbf{40.37} \\
& 1 & 24.98 & 26.28 & 25.41 & 26.68 & 33.25 & \textbf{35.69} \\
& 2 & 20.29 & 20.29 & 21.97 & 24.25 & 32.06 & \textbf{34.84} \\
\midrule
\multirow{3}{*}{0.3} 
& 0 & 41.87 & 45.43 & 42.64 & 45.72 & 46.21 & \textbf{52.56} \\
& 1 & 37.47 & 39.88 & 39.77 & 41.34 & 47.28 & \textbf{54.01} \\
& 2 & 34.46 & 34.73 & 34.64 & 34.66 & 40.40 & \textbf{48.50} \\
\midrule
\multirow{3}{*}{0.5} 
& 0 & 37.81 & 41.52 & 38.13 & 39.04 & 42.63 & \textbf{43.37} \\
& 1 & 39.42 & 39.49 & 40.46 & 42.82 & 46.69 & \textbf{49.91} \\
& 2 & 40.87 & 45.05 & 41.64 & 42.78 & 45.68 & \textbf{55.06} \\
\bottomrule
\end{tabular}
\end{table}

\begin{table}[H]
\caption{Test accuracy of the server model of different methods in OfficeHome in three random seeds and across five levels of statistical heterogeneity (lower $\alpha$ is more heterogeneous).} 
\label{tab_5}
\centering
\small
\begin{tabular}{@{}c|c|cccccc@{}}
\toprule
$p \sim Dir(\alpha)$  
& seed & FedAvg & FedFTG & DENSE & Co-Boosting & DeepInversion & FedMITR  \\ \midrule
\multirow{3}{*}{0.01} 
& 0 & 9.16 & 9.51 & 9.41 & 9.69 & 12.72 & \textbf{25.19} \\
& 1 & 6.48 & 7.08 & 7.01 & 7.26 & 10.13 & \textbf{22.82} \\
& 2 & 8.54 & 9.79 & 9.45 & 9.51 & 12.78 & \textbf{24.13} \\
\midrule
\multirow{3}{*}{0.05} 
& 0 & 14.84 & 15.27 & 15.02 & 15.37 & 19.11 & \textbf{30.49} \\
& 1 & 11.78 & 12.94 & 12.69 & 13.12 & 17.64 & \textbf{28.40} \\
& 2 & 14.43 & 15.62 & 15.27 & 15.62 & 19.42 & \textbf{31.23} \\
\midrule
\multirow{3}{*}{0.1} 
& 0 & 20.39 & 21.26 & 21.01 & 21.42 & 25.44 & \textbf{34.60} \\
& 1 & 16.08 & 17.49 & 17.08 & 17.49 & 22.69 & \textbf{32.01} \\
& 2 & 16.43 & 18.55 & 17.39 & 17.55 & 23.53 & \textbf{31.83} \\
\midrule
\multirow{3}{*}{0.3} 
& 0 & 27.90 & 29.21 & 28.74 & 29.36 & 34.98 & \textbf{35.82} \\
& 1 & 27.65 & 29.80 & 29.40 & 29.11 & 33.95 & \textbf{35.19} \\
& 2 & 25.09 & 27.37 & 26.59 & 27.34 & 31.70 & \textbf{35.57} \\
\midrule
\multirow{3}{*}{0.5} 
& 0 & 34.07 & 36.07 & 35.50 & 35.88 & \textbf{41.33} & 40.02 \\
& 1 & 30.33 & 32.01 & 31.89 & 32.29 & \textbf{37.41} & 37.00 \\
& 2 & 29.02 & 30.49 & 30.05 & 30.55 & 34.88 & \textbf{37.44} \\
\bottomrule
\end{tabular}
\end{table}

\begin{table}[H]
\caption{Test accuracy of the server model of different methods in Mini-ImageNet in three random seeds and across five levels of statistical heterogeneity (lower $\alpha$ is more heterogeneous).} 
\label{tab_6}
\centering
\small
\begin{tabular}{@{}c|c|cccccc@{}}
\toprule
$p \sim Dir(\alpha)$  
& seed & FedAvg & FedFTG & DENSE & Co-Boosting & DeepInversion & FedMITR  \\ \midrule
\multirow{3}{*}{0.01} 
& 0 & 12.01 & 12.47 & 12.53 & 13.14 & 20.86 & \textbf{40.83} \\
& 1 & 14.33 & 15.35 & 14.76 & 15.17 & 22.05 & \textbf{50.89} \\
& 2 & 15.62 & 16.36 & 16.18 & 16.93 & 24.56 & \textbf{44.07} \\
\midrule
\multirow{3}{*}{0.05} 
& 0 & 40.41 & 41.20 & 40.89 & 41.74 & 50.49 & \textbf{62.26} \\
& 1 & 36.26 & 37.49 & 37.28 & 37.19 & 44.82 & \textbf{62.36} \\
& 2 & 37.28 & 38.07 & 37.79 & 38.17 & 46.72 & \textbf{61.83} \\
\midrule
\multirow{3}{*}{0.1} 
& 0 & 54.87 & 55.69 & 55.33 & 55.83 & 62.71 & \textbf{69.94} \\
& 1 & 50.92 & 52.53 & 51.91 & 52.22 & 59.19 & \textbf{68.67} \\
& 2 & 51.64 & 52.65 & 52.32 & 52.36 & 58.95 & \textbf{66.01} \\
\midrule
\multirow{3}{*}{0.3} 
& 0 & 75.78 & 76.69 & 76.19 & 76.49 & \textbf{78.84} & 76.58 \\
& 1 & 78.81 & 79.32 & 79.28 & 79.37 & \textbf{80.47} & 79.17 \\
& 2 & 75.92 & 76.90 & 76.48 & 76.51 & \textbf{77.88} & 76.56 \\
\midrule
\multirow{3}{*}{0.5} 
& 0 & 81.83 & 82.09 & 82.08 & 82.09 & \textbf{82.42} & 81.93 \\
& 1 & 82.12 & 82.45 & 82.24 & 82.43 & \textbf{83.93} & 82.33 \\
& 2 & 81.67 & 82.22 & 82.17 & 82.10 & \textbf{82.93} & 82.27 \\
\bottomrule
\end{tabular}
\end{table}

\begin{table}[H]
\caption{Test accuracy of the server model of different methods in Mini-ImageNet in three random seeds under extreme heterogeneity setting.} 
\label{tab_7}
\centering
\small
\begin{tabular}{@{}c|c|cccccc@{}}
\toprule
$\#C=k$ 
& seed & FedAvg & FedFTG & DENSE & Co-Boosting & DeepInversion & FedMITR  \\
\midrule
\multirow{3}{*}{$\#C=10$ } 
& 0 & 7.37 & 7.80 & 7.77 & 7.82 & 12.89 & \textbf{20.72} \\
& 1 & 8.24 & 9.54 & 9.02 & 9.24 & 14.72 & \textbf{23.13} \\
& 2 & 7.78 & 8.30 & 8.05 & 8.12 & 13.15 & \textbf{21.57} \\
\bottomrule
\end{tabular}
\end{table}

\subsection{Proof of Sparse Model Inversion and Token Relabel: Generalization via Stability in Vision Transformers}\label{sec:proof_of_model_inversion_and_token_relabel}

We analyze the generalization error using \textbf{Algorithmic Stability} \cite{bousquet2002stability}. The generalization gap is bounded by the algorithm's uniform stability $\beta$. For SGD, $\beta$ is controlled by the Lipschitz constant of the loss function gradients \cite{hardt2016train}. We prove that FedMITR yields a smaller Lipschitz constant than Dense Inversion (DI) by accounting for the non-linear interaction in Vision Transformers (ViTs).

We begin by outlining several fundamental concepts related to generalization error, followed by a detailed proof showing that the operations of Sparse Model Inversion and Token Relabel reduce the upper bound of the generalization error associated with Dense Model Inversion.

\vspace{1em}
\noindent
\textbf{\ref{sec:proof_of_model_inversion_and_token_relabel}.1. Definitions and Lemmas}

\setcounter{assumption}{0}
\setcounter{definition}{0}
\setcounter{lemma}{0}
\setcounter{theorem}{0}

\begin{definition}[Generalization Error]
Let $D$ be the unknown data distribution and $S = \{z_1, \dots, z_N\}$ be a training set of $N$ samples drawn i.i.d. from $D$. Let $A(S)$ denote the hypothesis (model parameters $\theta$) output by stochastic algorithm $A$ on dataset $S$. The generalization error (gap) is defined as:
\begin{equation}
    \epsilon_{gen} \triangleq \mathbb{E}_{S, A} [R(A(S)) - R_{emp}(A(S))]
\end{equation}
where $R(\theta) = \mathbb{E}_{z \sim D}[\mathcal{L}(\theta; z)]$ is the expected risk and $R_{emp}(\theta) = \frac{1}{N} \sum_{i=1}^N \mathcal{L}(\theta; z_i)$ is the empirical risk.
\end{definition}

\begin{definition}[Uniform Stability \cite{hardt2016train}]
A stochastic algorithm $A$ is $\beta$-uniformly stable if, for all datasets $S, S'$ differing by at most one example, and for all $z$, the following inequality holds:
\begin{equation}
    \sup_{z} \mathbb{E}_A [\mathcal{L}(A(S); z) - \mathcal{L}(A(S'); z)] \le \beta
\end{equation}
Uniform stability implies that the algorithm's output does not depend too heavily on any single training example.
\end{definition}

\begin{lemma}[Generalization Bound via Stability \cite{mohri2018foundations,hardt2016train}]
\label{lemma:generalization_and_stability}
Let algorithm $A$ be $\beta$-uniformly stable. Assume the loss function is bounded such that $0 \le \mathcal{L} \le M$. Then, for any $\delta \in (0, 1)$, with probability at least $1 - \delta$, the following bound holds:
\begin{equation}
    R(A_S) \le R_{emp}(A_S) + 2\beta + (4N\beta + M)\sqrt{\frac{\ln(1/\delta)}{2N}}
\end{equation}
where $R(A_S)$ is the expected risk (test error), $R_{emp}(A_S)$ is the empirical risk (training error), and $N$ is the number of training samples.
\end{lemma}

To analyze the stability parameter $\beta$ in the context of our Vision Transformer (ViT) optimization, we introduce the following assumptions regarding the loss landscape and the manifold consistency of the synthetic data.

\begin{assumption}[Boundedness and Lipschitz Smoothness]
\label{ass:lip}
For any input $X$ and parameter $\theta$, the loss function $\mathcal{L}(\theta; X, y)$ satisfies:
\begin{enumerate}
    \item \textbf{Boundedness:} $\mathcal{L}(\theta; z) \in [0, 1]$. (Note: Losses can be scaled to satisfy this).
    \item \textbf{$L$-Lipschitz Continuity:} The gradient norm is bounded by $L$:
    $\sup_{\theta} \| \nabla \mathcal{L}(\theta) \| \le L$.
    \item \textbf{$\mu$-Smoothness:} The loss function has $\mu$-Lipschitz continuous gradients (often denoted as $\beta$-smooth in optimization literature):
    $\| \nabla \mathcal{L}(\theta_1) - \nabla \mathcal{L}(\theta_2) \| \le \mu \| \theta_1 - \theta_2 \|$.
\end{enumerate}
\end{assumption}

\begin{assumption}[Statistical Orthogonality of Background Noise]
\label{ass:orthogonality}
Due to the ill-posed nature of the model inversion problem, where a low-dimensional label is mapped back to a high-dimensional input space, the optimization process tends to fill unconstrained regions (backgrounds) with high-frequency noise rather than semantic features.
Let $X_l$ denote the feature vectors of these background tokens and $y$ be the semantic direction of the target class label in the embedding space $\mathbb{R}^D$. We assume:
\begin{enumerate}
    \item \textbf{Noise Distribution:} $X_l$ follows a high-dimensional isotropic distribution (e.g., Gaussian noise) driven by initialization and gradient stochasticity, containing no transferable semantic information regarding $y$.
    \item \textbf{Orthogonality:} In high-dimensional spaces ($D \gg 1$), random vectors are asymptotically orthogonal to any fixed direction due to the concentration of measure phenomenon. Thus, $X_l$ is statistically orthogonal to $y$:
    \begin{equation}
        \mathbb{E}[\langle f(X_l), y \rangle] \approx 0 \quad \text{and} \quad \text{Cov}(f(X_l), y) \approx 0
    \end{equation}
\end{enumerate}
This implies that the hard label $y$ provides no predictive signal for the noise $X_l$, rendering the error signal $\|p(X_l) - y\|$ consistently large.
\end{assumption}

\begin{lemma}[Stability of SGD for Non-convex Losses \cite{hardt2016train}]
\label{ass:stability_of_sgd}
Suppose the loss function satisfies Assumption 1 ($L$-Lipschitz and $\mu$-smooth). Let algorithm $A$ be Stochastic Gradient Descent (SGD) running for $T$ steps with monotonically non-increasing step sizes $\alpha_t \le c/t$. Then, the algorithm satisfies uniform stability with stability parameter $\beta$ bounded by:
\begin{equation}\label{eq:stability_bound}
    \beta \le \frac{1 + 1/(\mu c)}{N-1} \left( 2 c L^2 \right)^{\frac{1}{\mu c + 1}} T^{\frac{\mu c}{\mu c + 1}}
\end{equation}
where $N$ is the sample size and $c$ is constant.
\end{lemma}

\begin{remark}\label{remark1}
This theorem is crucial for our analysis as it explicitly connects the stability bound $\beta$ to the gradient Lipschitz constant $L$ in a non-convex setting.
Observing Eq. (\ref{eq:stability_bound}), the stability bound depends on the gradient norm $L$ through the term $(2cL^2)^{\frac{1}{\mu c + 1}}$. Since the exponent $\frac{1}{\mu c + 1}$ is positive, the stability bound $\beta$ is a strictly increasing function of $L$.
$$ L_{Fed} < L_{DI} \implies \beta_{Fed} < \beta_{DI} $$
Here, $L_{Fed}$ denotes the gradient upper bound of the loss function in FedMITR, while $L_{DI}$ represents the gradient upper bound of the loss function in Dense Inversion.
Therefore, by establishing that FedMITR decreases the gradient upper bound $L$, we obtain a mathematically tighter generalization bound for the non-convex optimization of ViTs. Consequently, this implies that the Sparse Model Inversion and Token Relabel steps in our method further improve generalization performance.
\end{remark}

\vspace{1em}
\noindent
\textbf{\ref{sec:proof_of_model_inversion_and_token_relabel}.2. Theoretical Analysis of Gradient Lipschitz Constants}

Motivated by Lemma \ref{ass:stability_of_sgd} and Remark \ref{remark1}, we need to prove that FedMITR possesses a strictly smaller gradient upper bound (Lipschitz constant) than Dense Inversion. We perform this analysis by tracing the gradient backpropagation within a Vision Transformer (ViT) layer and applying statistical learning theory results regarding knowledge distillation \cite{menon2021statistical}.

\subsubsection{Gradient Formulation in Vision Transformers}

Let the input sequence $X \in \mathbb{R}^{N \times D}$ be partitioned into a set of semantic foreground indices $\mathcal{I}_h$ and background noise indices $\mathcal{I}_l$. The Multi-Head Self-Attention (MSA) layer computes the output $Z$ as:
\begin{equation}
    Z = \text{MSA}(X) = A X W_V, \quad \text{where } A = \text{Softmax}\left(\frac{X W_Q W_K^T X^T}{\sqrt{d}}\right)
\end{equation}
Here, $A \in \mathbb{R}^{N \times N}$ is the attention matrix and $W_V$ is the value projection weight. We focus on the gradient with respect to $W_V$, as it dictates how token content influences optimization. Applying the matrix chain rule, the gradient of the loss $\mathcal{L}$ with respect to $W_V$ is the outer product of the input and the backpropagated error signal:
\begin{equation}
    \label{eq:grad_chain}
    \nabla_{W_V} \mathcal{L} = X^T A^T \delta
\end{equation}
where $\delta \triangleq \nabla_Z \mathcal{L} \in \mathbb{R}^{N \times D}$ is the error signal at the layer output.
The Lipschitz constant $L$ is defined by the supremum of the gradient norm: $L \triangleq \sup_{\theta, X} \| \nabla_\theta \mathcal{L} \|$. Thus, stability analysis reduces to comparing the expected squared Frobenius norm of the gradient contributions from background noise tokens.

\subsubsection{Instability of Dense Inversion (DI)}

In Dense Inversion, the objective is the cross-entropy loss against a hard label $y$: $\mathcal{L}_{DI} = \ell_{\text{CE}}(f(X), y)$. The error signal is $\delta_{DI} = p(X) - y$, where $p(X)$ is the softmax probability.
We analyze the gradient contribution from background noise tokens ($j \in \mathcal{I}_l$). Substituting into Eq. (\ref{eq:grad_chain}):
\begin{equation}
    \nabla_{W_V}^{(noise)} \mathcal{L}_{DI} = \sum_{j \in \mathcal{I}_l} (X_j)^T (A^T)_{:,j} \delta_{DI}
\end{equation}
Under Assumption \ref{ass:orthogonality} (Statistical Orthogonality), the background $X_l$ represents noise features independent of $y$.
\begin{itemize}
    \item Since $X_l$ is noise, the model's initial prediction $p(X)$ is uncorrelated with $y$.
    \item The hard label $y$ acts as a high-variance random variable relative to the noise input.
\end{itemize}
The expected gradient norm is dominated by the variance of the hard label mismatch. For a $K$-class problem, since $y$ is a one-hot vector and $p(X)$ is near uniform for noise, this term saturates: $\mathbb{E}[\|\delta_{DI}\|^2] \approx 1 - 1/K$. Consequently, fitting noise to hard labels forces the optimizer to "memorize" the noise, resulting in an increased estimation variance for the same image under noise, which in turn yields a large Lipschitz constant $L_{DI}$.

\subsubsection{Stability of FedMITR}

FedMITR decouples the loss function:
\begin{equation}
    \mathcal{L}_{Fed} = \underbrace{\ell_{\text{CE}}(f(X_h), y)}_{\text{Term 1: Sparsity}} + \lambda \underbrace{\ell_{\text{KL}}(f(X_l), E(X_l))}_{\text{Term 2: Relabeling}}
\end{equation}
We prove that FedMITR strictly reduces the gradient norm via two mechanisms: gradient elimination (Sparsity) and gradient variance reduction (Relabeling).

\textbf{Mechanism 1: Sparsity as Gradient Elimination.}
For Term 1, we apply a binary mask $M$ such that the effective input $\tilde{X}=M\odot X$, where $\tilde{X}_j = 0$ for noise indices $j \in \mathcal{I}_l$. Substituting this into Eq. (\ref{eq:grad_chain}):
\begin{equation}
    \nabla_{W_V}^{(noise)} \mathcal{L}_{\text{Term 1}} = (\tilde{X}_j)^T A^T \delta = \mathbf{0}^T A^T \delta = 0
\end{equation}
By physically setting the input to zero, Sparsity breaks the backpropagation link. The large gradient contribution from the noise-hard-label mismatch identified in the DI analysis is mathematically eliminated. This directly reduces the norm of the gradient.

\textbf{Mechanism 2: Relabeling as Variance Reduction.}
For Term 2, the error signal is $\delta_{Fed} = p(X_l) - q(E(X_l))$, where $q$ is the soft target from the ensemble teacher. To quantify the reduction in gradient norm compared to hard labels, we leverage the theoretical finding from Menon et al. \cite{menon2021statistical}.

\begin{lemma}[Gradient Variance Decomposition \cite{menon2021statistical}]
\label{lemma:menon_variance}
Let $\mathcal{L}(f(x), y)$ be the loss with respect to a target $y$. The expected squared norm of the stochastic gradient can be decomposed into the squared norm of the mean gradient and the trace of the gradient covariance (variance):
\begin{equation}
    \mathbb{E}[\| \nabla \mathcal{L} \|^2] = \| \mathbb{E}[\nabla \mathcal{L}] \|^2 + \text{Tr}(\text{Var}(\nabla \mathcal{L}))
\end{equation}
Menon et al. \cite{menon2021statistical} demonstrate that using soft labels (distillation) significantly reduces the variance term compared to using hard labels, as the soft labels provide a lower-variance estimator of the true class probability. Specifically:
\begin{equation}
    \text{Tr}(\text{Var}(\nabla \mathcal{L}_{soft})) \ll \text{Tr}(\text{Var}(\nabla \mathcal{L}_{hard}))
\end{equation}
\end{lemma}

\textit{Application to FedMITR:}
In the specific case of background noise tokens $X_l$:
\begin{itemize}
    \item \textbf{Hard Labels:} The label $y$ is a fixed one-hot vector that is statistically orthogonal to the input $X_l$. The gradient $\nabla \mathcal{L}_{hard}$ attempts to enforce a correlation that does not exist, resulting in a high-magnitude error signal ($\delta \approx 1$) which acts as a large variance source in the optimization landscape.
    \item \textbf{Soft Labels:} The teacher's prediction $q = E(X_l)$ is consistent with the student's inductive bias (as both are models processing the same noise). The gradient $\nabla \mathcal{L}_{soft}$ drives the student to match the teacher's response. Since the student $p$ and teacher $q$ lie on similar manifolds, the error term $p-q$ is small, and the variance of this gradient is minimized.
\end{itemize}
Therefore, applying Lemma \ref{lemma:menon_variance}, we obtain:
\begin{equation}
    \mathbb{E}[\| \nabla_{W_V}^{(noise)} \mathcal{L}_{\text{Term 2}} \|^2] \ll \mathbb{E}[\| \nabla_{W_V}^{(noise)} \mathcal{L}_{DI} \|^2]
\end{equation}
This indicates that Relabeling transforms the optimization from a high-variance random walk (matching noise to hard labels) into a stable convergence trajectory (matching teacher consistency).

\subsubsection{Main Theorem and Conclusion}

We formally establish the inequality of Lipschitz constants.



\begin{theorem}[Lipschitz Constant Reduction]
\label{thm:lipschitz_reduction}
Let $L_{DI}$ and $L_{Fed}$ be the Lipschitz constants of the loss gradients with respect to model parameters for Dense Inversion and FedMITR, respectively. Under Assumption \ref{ass:lip}, Assumption \ref{ass:orthogonality}, and Lemma \ref{lemma:menon_variance}, we have:
\begin{equation}
    L_{Fed} < L_{DI}
\end{equation}
\end{theorem}

\begin{proof}
The Lipschitz constant $L$ is defined by the supremum of the gradient norm with respect to the parameters $\theta$. Considering the gradient backpropagation in the MSA layer derived in Eq. (\ref{eq:grad_chain}), the gradient norm is bounded by the product of the input norm, the attention map norm, and the error signal norm:
\begin{equation}
    \| \nabla_{W_V} \mathcal{L} \| = \| X^T A^T \delta \| \le \| X \| \cdot \| A \| \cdot \| \delta \|
\end{equation}
We decompose the total Lipschitz constant $L$ based on the contributions from foreground ($\mathcal{I}_h$) and background noise ($\mathcal{I}_l$) tokens.

\textbf{1. Analysis of Dense Inversion ($L_{DI}$):}
The noise component is driven by the interaction between the high-energy noise input $X_{noise}$ and the hard-label error signal $\delta_{hard}$.
$$ L_{DI}^{(noise)} \propto \sup \left( \| X_{noise} \| \cdot \| \delta_{hard} \| \right) $$
Under Assumption \ref{ass:orthogonality}, $\|\delta_{hard}\|$ is saturated (large) due to orthogonality. Thus, the product $\|X_{noise}\| \cdot \|\delta_{hard}\|$ constitutes a large gradient upper bound.

\textbf{2. Analysis of FedMITR ($L_{Fed}$):}
The noise component is decoupled into two terms:
\begin{itemize}
    \item \textbf{Sparsity Term:} The effective input is masked to zero ($\tilde{X} = 0$). Consequently, the backpropagated gradient norm is:
    $$ \| \tilde{X}^T A^T \delta_{hard} \| = \| 0 \cdot A^T \delta_{hard} \| = 0 $$
    \item \textbf{Relabeling Term:} The input is $X_{noise}$, but the error signal is $\delta_{soft}$. The gradient norm is:
    $$ L_{Fed}^{(relabel)} \propto \lambda \sup \left( \| X_{noise} \| \cdot \| \delta_{soft} \| \right) $$
\end{itemize}

\textbf{3. Comparison:}
Combining the components, we compare the upper bounds:
$$ L_{DI} \approx L_{fg} + C \cdot \mathbb{E}[\|\delta_{hard}\|] $$
$$ L_{Fed} \approx L_{fg} + 0 + \lambda C \cdot \mathbb{E}[\|\delta_{soft}\|] $$
where $C$ represents the common factor involving input and attention norms.
According to Lemma \ref{lemma:menon_variance}, the expected norm of the error signal from soft labels is strictly smaller than that from hard labels due to variance reduction: $\mathbb{E}[\|\delta_{soft}\|] \ll \mathbb{E}[\|\delta_{hard}\|]$.
Therefore, strictly:
$$ L_{Fed} < L_{DI} $$
\end{proof}

\vspace{0.5em}
\noindent
\textbf{Conclusion via Stability.}
Substituting the result $L_{Fed} < L_{DI}$ into the stability bound of SGD (Eq. \ref{eq:stability_bound}), we obtain $\beta_{Fed} < \beta_{DI}$. Finally, invoking Lemma \ref{lemma:generalization_and_stability}, we conclude that FedMITR achieves a tighter generalization bound. This theoretical derivation confirms that our method effectively stabilizes the non-convex optimization of ViTs by eliminating spurious noise gradients and enforcing consistent teacher-student dynamics. Table \ref{tab_main} in the experimental section clearly demonstrates the improvement in generalization achieved by FedMITR.

\end{document}